\documentclass[11pt]{article}

\usepackage{sn-preamble}

\newcommand{\ignore}[1]{}

\newcommand{\diam}{\mathrm{diam}}

\renewcommand{\S}{\mathbb{S}}

\newcommand{\dG}{\dist_{\mathrm{G}}}
\newcommand{\Aux}{\mathrm{Aux}}
\newcommand{\Median}{\mathrm{med}}

\title{Sparsifying Suprema of Gaussian Processes}
\author{
Anindya De\thanks{University of Pennsylvania. Email: \url{anindyad@seas.upenn.edu}.}
\and 
Shivam Nadimpalli\thanks{MIT. Email: \url{shivamn@mit.edu}.}
\and 
Ryan O'Donnell\thanks{Carnegie Mellon University. Email: \url{odonnell@cs.cmu.edu}.}
\and 
Rocco A. Servedio\thanks{Columbia University. Email: \url{ras2105@columbia.edu}.}
\vspace{0.5em}}
\date{\today}

\hypersetup{
	colorlinks = true,
    linkcolor={purple},
    urlcolor={purple},
    citecolor={blue!80!black}
}

\begin{document}

\hypersetup{linkcolor={black}}

\pagenumbering{gobble}

\maketitle

\hypersetup{linkcolor={purple}}

\begin{abstract}
We give a dimension-independent sparsification result for suprema of centered Gaussian processes: Let $T$ be any (possibly infinite) bounded set of vectors in $\R^n$, and let $\{\bX_t  := t \cdot \boldsymbol{g} \}_{t\in T}$ be the canonical Gaussian process on $T$, where $\boldsymbol{g}\sim N(0, I_n)$. We show that there is an $O_\varepsilon(1)$-size subset $S \subseteq T$ and a set of real values $\{c_s\}_{s \in S}$ such that the random variable $\sup_{s \in S} \{{\boldsymbol{X}}_s + c_s\}$ is an $\varepsilon$-approximator\,(in $L^1$) of the random variable $\sup_{t \in T} {\boldsymbol{X}}_t$. 
Notably, the size of the sparsifier $S$ is completely independent of both $|T|$ and the ambient dimension $n$.

We give two applications of this sparsification theorem:
\begin{itemize}

	\item \textbf{A ``Junta Theorem'' for Norms:} 
	We show that given any norm $\nu(x)$ on $\R^n$, there is another norm $\psi(x)$ depending only on the projection of $x$ onto $O_\varepsilon(1)$ directions, for which $\psi({\boldsymbol{g}})$ is a multiplicative $(1 \pm \varepsilon)$-approximation of $\nu({\boldsymbol{g}})$ with probability $1-\varepsilon$ for ${\boldsymbol{g}} \sim N(0,I_n)$. 
	
	\item \textbf{Sparsification of Convex Sets:} 
	We show that any intersection of (possibly infinitely many) halfspaces in $\R^n$ that are at distance $r$ from the origin is $\varepsilon$-close (under $N(0,I_n)$) to an intersection of only $O_{r,\varepsilon}(1)$ halfspaces. 
	This yields new polynomial-time \emph{agnostic learning} and \emph{tolerant property testing} algorithms for intersections of halfspaces. 
	
\end{itemize}
\end{abstract}

\newpage

\pagenumbering{arabic}


\section{Introduction}
\label{sec:intro}

A recurring theme in contemporary mathematics and theoretical computer science is the study of how complex or high-dimensional objects can be usefully approximated by simpler or lower-dimensional objects, while preserving properties of interest.
Results of this general sort, which are sometimes called \emph{sparsification theorems}, are of interest across many areas.
A partial and incomplete list of mathematical objects that have been ``sparsified'' in this way includes graphs~\cite{benczur1996approximating,spielman2008graph,BSST13,Srivastava14}, hypergraphs~\cite{BST19,soma2019spectral,kapralov2021towards,lee2023spectral,jambulapati2023chaining}, matrices~\cite{AM07}, logical formulas~\cite{IPZ98,GMR13,lovett2021decision}, constraint satisfaction problems~\cite{KK15,KPS24}, sums of norms and symmetric submodular functions~\cite{JLLS23}, and many more. 
These sparsification results have contributed to applications across a wide range of diverse areas, including but not limited to near-linear-time graph algorithms, satisfiability algorithms, unconditional pseudorandomness, numerical linear algebra, and sketching algorithms.

This work concerns sparsification of a broad and fundamental class of random variables:~\emph{suprema of Gaussian processes}. 

\paragraph{Background: Gaussian Processes.} 

Recall that a \emph{Gaussian process} is a joint distribution \smash{$\{\bX_t\}_{t \in T}$} over (possibly infinitely many) random variables such that every finite subset of them is jointly Gaussian.  
Gaussian processes are central objects of study in probability theory~\cite{Mil:71,dudley1967sizes, sudakov1969gauss, Fernique75, og-mm,gordon1988milman, gordon1992majorization,talagrand1995sections,  van2018chaining}, and their relevance to theoretical computer science and data science has become increasingly clear over the past decade or so; see e.g.~\cite{stojnic2009various,DLP12,oymak2010new,plan2012robust,stojnic2013regularly,Meka15,Nelson16,ORS18,BST19,BDOS21}. 
We refer the reader to \cite{rvh-notes,vershynin2018high,talagrand2022upper} for further background on Gaussian processes and their applications. 

A canonical way to generate a Gaussian process is through a bounded set of vectors $T \sse \R^n$. 
Define 
\[
	\bX_t := t\cdot \bg~\text{for}~t \in T,~\text{where}~\bg \sim N(0,I_n). 
\]
Here $N(0, I_n)$ denotes the $n$-dimensional standard Gaussian distribution. 
The resulting process $\{\bX_t\}_{t\in T}$ is called the \emph{canonical Gaussian process} on $T$. 
It is a standard fact that every finite centered Gaussian process (that is, one where each component has mean $0$) can be represented as a canonical process. 
Hence, in what follows, we restrict attention to canonical Gaussian processes without loss of generality. 

It is easily verified (see, for example, Remark~7.1.11 of \cite{vershynin2018high}) that the covariance function $\E[\bX_t\bX_s]$ for $s, t \in T$ uniquely determines the law of the centered Gaussian process $\{\bX_t\}_{t\in T}$, generalizing the familiar fact that a multivariate Gaussian distribution is determined by its mean and covariance matrix. 
Moreover, since $\E[\bX_t\bX_s] = t\cdot s$, this covariance structure both arises from and encodes the Euclidean distances $\|t-s\|_2$. 
In this precise sense, the geometry of $T$ is in one-to-one correspondence with the law of the process $\{\bX_t\}_{t\in T}$. 

\paragraph{Suprema of Gaussian Processes.} 

Much of the interest in Gaussian processes centers around understanding the behavior of their \emph{suprema}---that is, the random variable $\sup_{t\in T} \bX_t$ associated with a Gaussian process $\{\bX_t\}_{t\in T}$. 
This quantity is fundamental: it arises naturally in a wide range of settings, including compressed sensing, convex analysis, learning theory, and random matrix theory (see Section 1.2.2 of~\cite{rvh-notes} for examples). 

Since the index set $T$ of a Gaussian process $\{\bX_t\}_{t \in T}$ may be arbitrarily large or even infinite, the supremum of $\{\bX_t\}_{t \in T}$ can be quite difficult to understand. 
A crowning achievement in the study of Gaussian processes is the theory of \emph{majorizing measures}, developed by Fernique, Talagrand, and others, which provides a characterization of the expected supremum $\E[\sup_{t\in T} \bX_t]$ up to constant factors through a hierarchical clustering (called the ``generic chaining'') of the index set $T$ (see \cite{Fernique75,og-mm,LedouxTalagrand,talagrand2022upper} and the references therein). 
This result and the associated technique of generic chaining have found diverse applications in theoretical computer science~\cite{DLP12,meka2012ptas,rudra2014every,Nelson16,braverman2016beating,diakonikolas2025sos,bartl2025uniform}. 

The theory of majorizing measures thus offers a geometric characterization of the expected supremum $\E[\sup_{t\in T} \bX_t]$ up to constant factors. 
In contrast, in this work our goal is to approximate the random variable $\sup_{t\in T} \bX_t$ itself (in an $L^1$ sense) rather than merely its expectation.
Our main result achieves this by providing a sparse geometric representation derived from Talagrand's majorizing measures theorem. 


\paragraph{This Work: Sparsifying Suprema of Gaussian Processes.} 

The main contribution of this work is a \emph{dimension-independent} sparsification result for the suprema of centered Gaussian processes: 

\begin{restatable}{theorem}{theoremsparsification}\label{thm:GP-sparsification}	
	Let $T\subset\R^n$ be a bounded set of vectors and let $\{\bX_t\}_{t\in T}$ be the canonical Gaussian process on $T$. For any $\eps > 0$, there is a subset $S \sse T$ with 
	\[
		|S| = 2^{2^{O\pbra{\frac{1}{\epsilon}}}}
	\]
	and a collection of real values $\{c_s\}_{s \in S}$ such that
	\begin{equation} \label{eq:glory}
		\Ex\sbra{\abs{\sup_{t \in T} \bX_t - \sup_{s\in S} \cbra{\bX_s + c_s}}} \leq \eps \cdot \Ex\sbra{\sup_{t \in T} \bX_t}.
	\end{equation}
\end{restatable} 

\medskip 

In other words, we show that there is a constant-size \emph{non-centered} Gaussian process $\{\bX_s + c_s\}_{s \in S}$, where $S \subseteq T$, that is a high-accuracy approximator (in an $L^1$-sense) of the original Gaussian process. 
In particular, by applying Markov's inequality to \Cref{eq:glory} we get that the supremum of the ``sparsified'' process $\sup_{s \in S} \bX_s + c_s$ is additively close to $\sup_{t \in T} \bX_t$ with high probability, with an error on the order of $\E[\sup_{t \in T} \bX_t]$, the natural scale of the process. 

With some additional work, it is possible to obtain a \emph{centered} approximator to the original Gaussian process, although the vectors defining this approximator no longer lie in the original set $T$; we record this as~\Cref{cor:shifted-shinboku}.  

In the rest of this introduction, in \Cref{subsec:apps} we first describe  some applications of \Cref{thm:GP-sparsification}.
In \Cref{subsec:techniques} we give a technical overview of the proof of \Cref{thm:GP-sparsification}, which relies on Talagrand's majorizing measures theorem~\cite{og-mm,talagrand2022upper}  mentioned above (the full proof is given in~\Cref{sec:shinboku}). 
In \Cref{sec:discussion-applications}  we provide some context and discussion of our results.  
In \Cref{subsec:rel} we describe the connection between \Cref{thm:GP-sparsification} and some of the many classical and modern results that relate the expected supremum $\E[\sup_{t\in T}\bX_t]$ to the geometry, complexity, or dimensionality of the index set $T$, including the theory of majorizing measures~\cite{og-mm,talagrand2022upper} as well as results in dimensionality reduction~\cite{gordon1988milman,klartag2005empirical}. 


\subsection{Applications of \Cref{thm:GP-sparsification} }
\label{subsec:apps}

We now turn to two consequences of \Cref{thm:GP-sparsification}. 

\subsubsection{A Junta Theorem for Norms} 

We obtain a ``junta theorem'' for norms over Gaussian space as a consequence of \Cref{thm:GP-sparsification}. 
As a motivating example, if $\nu(x) = \|x\|$ is the $\ell_2$ norm on $\R^n$, then it is easy to see via standard concentration bounds  (see, for example,~\Cref{prop:lipschitz-concentration}) that the norm $\psi:\R^n\to\R$ given by   
\[
	\psi(x) = \eps\sqrt{n}\pbra{\sup_{1\leq i \leq 2^{\Theta(1/\eps^2)}}{|x_i|}}
\]
multiplicatively approximates $\nu(x)$ within a $(1\pm\eps)$-factor on $(1-\eps)$-fraction of $N(0, I_n)$.  
Note that $\psi(\cdot)$ only depends on $2^{\Theta(1/\eps^2)}$ directions, independent of the ambient dimension $n$. 
The following theorem says a similar phenomenon holds for \emph{every} norm over Gaussian space:

\begin{restatable}{theorem}{theoremnorm}
\label{thm:norm-junta-theorem}
	Fix any norm $\nu : \R^n \to \R$. For $\eps \in (0, 0.5)$, there exists another norm $\psi:\R^n\to\R$ for which
	\begin{equation} \label{eq:intro-norm-sparsification}
		\Prx_{\bg \sim N(0,I_n)}
		\sbra{
		1-\eps \leq \frac{\psi(\bg)}{\nu(\bg)} \leq 1 + \eps
		} 
		\geq 1-\eps
	\end{equation}
	where the norm $\psi:\R^n\to\R$ is a ``$\smash{2^{2^{O(1/\eps^{3})}}}$-subspace junta,'' i.e.~$\psi(x)$ depends only on the projection of the point $x$ onto a $\smash{2^{2^{O(1/\eps^{3})}}}$-dimensional subspace. 
\end{restatable}

\Cref{thm:norm-junta-theorem} is in fact an immediate consequence of a ``junta theorem'' for Gaussian processes whose index set $T$ is symmetric (i.e. $t\in T$ whenever $-t\in T$), namely~\Cref{cor:multiplicative-approximation}. (See~\Cref{remark:norm-junta} for more on this.)

\subsubsection{Sparsification of Convex Sets} 

We also use our sparsification result, \Cref{thm:GP-sparsification}, to obtain a general sparsification result for convex sets of bounded ``geometric width''. 
Formally, given two measurable sets $K, L \sse \R^n$, let 
\[
	\dG(K, L) := \Prx_{\bg\sim N(0,I_n)}\sbra{\bg \in K\,\triangle\,L}
\]
denote the \emph{Gaussian distance} between $K$ and $L$, where $K\,\triangle\,L = (K\setminus L) \cup (L\setminus K)$ is the symmetric difference of the sets $K$ and $L$. 
We prove the following: 

\begin{restatable}{theorem}{theorempolytope}
\label{thm:polytope-sparsification}
	Suppose $K\subset\R^n$ is an intersection of arbitrarily many halfspaces that are at distance at most $r \geq 1$ from the origin, i.e. there exists $T \sse \S^{n-1}$ and a collection of non-negative numbers $\{r_t\}_{t\in T}$ such that 
	\[
		K = \bigcap_{t \in T} \cbra{x \in \R^n : t\cdot x \leq r_t } \qquad \text{where}~ r_t \leq r~\text{for all}~t \in T.
	\] 
	For $0 < \eps < 0.5$, there exists $L\sse\R^n$ which is an intersection of 
	\[
		2^{\exp\pbra{\wt{O}\pbra{\frac{1}{\eps^2}}\cdot r^4}}~\text{halfspaces}
	\]such that $\dG(K, L) \leq \eps$. 
\end{restatable}

The class of convex sets which can be expressed as intersections of $k$ halfspaces in $\R^n$ has been extensively studied in Boolean function analysis, in particular from the vantage point of learning theory and property testing. It is known that over the Gaussian space, this class is both (a)~tolerantly testable with query complexity independent of the ambient dimension $n$~\cite{DMN21}, and (b)~agnostically learnable with a quasipolynomial dependence on $k$ and a polynomial dependence on $n$~\cite{KOS:08,DKKTZ23}. \Cref{thm:polytope-sparsification}, which shows that 
 bounded-width polytopes can be approximated by an intersection of $k$ halfspaces where $k$ depends only on the width and target error, lets us obtain qualitatively similar tolerant testing and agnostic learning results for bounded width polytopes over the Gaussian space, but with much stronger quantitative bounds. 
This is described in more detail in \Cref{sec:applications}. 

Finally, \Cref{thm:polytope-sparsification} can be viewed as a convex-set analogue of recently-established sparsification lemmas for ``narrow'' CNFs in the setting of Boolean functions~\cite{lovett2019dnf,lovett2021decision}; see \Cref{subsec:discussion-shinboku} for further discussion on this. 


\subsection{Techniques}
\label{subsec:techniques}

We begin by sketching the key ideas behind our main theorem, \Cref{thm:GP-sparsification}. To simplify the discussion, we first rescale the vectors so that $\mathbf{E}[\sup_{t \in T} \bX_t]=1$.  Now the goal 
is to show the existence of a subset $S \subseteq T$ and suitable constants $\{c_s\}_{s \in S}$, such that  
\[
	\mathbf{E}\sbra{ \abs{ \sup_{t \in T} \bX_t - \sup_{s \in S} \cbra{\bX_s + c_s} } } \le \epsilon\,. 
\]
A natural idea then is to 
use a clustering of the vectors in $T$. 
In particular, as a first attempt,
 consider a $\delta$-cover of $T$ for suitably small $\delta$, i.e.,~let $\mathcal{P}$ be a partition of $T$ such that each $P \in \mathcal{P}$ has (Euclidean) diameter $\delta$. For each part $P$ let $s_P$ be an arbitrarily chosen representative vector. Then, using the fact that the supremum of a Gaussian process has subgaussian tails (\Cref{fact:sup-has-subgaussian-tails}), it follows that 
\begin{equation}~\label{eq:error-single-cluster}
	\mathbf{E}\sbra{ \abs{ \sup_{t \in P} \bX_t - \cbra{\bX_{s_P} + c_{P}} } } = O(\delta) 
	\qquad\text{where}~
	c_{P} :=\mathbf{E} \sbra{\sup_{t \in P} \bX_t -\bX_{s_P}}.
\end{equation}
In words, \Cref{eq:error-single-cluster} means that within each part $P$, the supremum of the Gaussian process $\sup_{t \in P} \bX_t$ can be approximated by a single (non-centered) Gaussian random variable $\bX_{s_P} + c_{P}$. Crucially, the error in \Cref{eq:error-single-cluster} is just dependent on the diameter of $P$ and is independent of the number of vectors in $P$ and the ambient dimension of the vectors in $T$.
 
A natural way to approximate $\sup_{t \in T} \bX_t$ is to then take the supremum of the random variables $\bX_{s_P} + c_{P}$  as $P$ ranges over $\mathcal{P}$. Furthermore, another application of~\Cref{fact:sup-has-subgaussian-tails} shows that 
 \begin{equation}~\label{eq:total-error}
 \mathbf{E}\sbra{ \abs{ \sup_{t \in T} \bX_t - \sup_{P \in \mathcal{P}} \cbra{\bX_{s_P} + c_{P}} } } = O\pbra{\delta \sqrt{\log M}},
 \end{equation}
where $M$ is the number of parts in the partition $\mathcal{P}$.
Despite the fact that the error term only grows as $O(\sqrt{\log M})$, this dependence turns out to be inadequate for the above approach to work. 
In particular, there are Gaussian processes $\{\bX_t\}_{t \in T}$ such that $\mathbf{E}[\sup_{t \in T} \bX_t]=1$ and yet for every $\delta>0$, a $\delta$-covering of the space has size $\smash{M = 2^{\Theta(1/\delta^2)}}$, which means that the R.H.S. of~\Cref{eq:total-error} is an absolute constant.\footnote{We note that by the Sudakov ``minoration principle''~\cite{sudakov1969gauss} (alternatively, see Section~7.4 of~\cite{vershynin2018high}) and the fact that $\mathbf{E}[\sup_{t \in T} \bX_t]=1$, it follows that there is a $\delta$-cover of size $2^{O(1/\delta^2)}$.} 

To circumvent this barrier, we consider a {\em multiscale} clustering instead of using a clustering of the entire space with balls of a single radius $\delta$; this naturally leads us to Talagrand's celebrated \emph{majorizing measures} theorem~\cite{og-mm}. (The modern treatment of this theorem is in terms of the so-called {\em generic chaining}~\cite{talagrand2022upper}.) For a Gaussian process $\{\bX_t\}_{t \in T}$ with $\mathbf{E}[\sup_{t \in T} \bX_t]=1$, the generic chaining formulation of the majorizing measures theorem can be seen as giving a hierarchical clustering 
of the underlying set of vectors with two key properties: 
\begin{enumerate}
	\item The $h^{\text{th}}$ level of the clustering is a partition of the space with at most  $2^{2^h}$ parts. 
	\item For any $t \in T$ and any $h$, let $\mathcal{A}_h(t)$ be the piece of the partition in which $t$ lies. Then, for every $t \in T$, we have that $\sum_{h \ge 0} 2^{h/2} \cdot \diam(\mathcal{A}_h(t))=O(1)$. 
\end{enumerate}
Note that the above implies that for every $t$, as $h \rightarrow \infty,$ we have 
 $\diam(\mathcal{A}_h(t)) = 2^{-h/2} \cdot o_h(1)$.
However, the rate of convergence of the sequence $\diam(\mathcal{A}_h(t)) \cdot 2^{h/2}$ can depend on $t$. If the convergence were uniform (i.e. independent of $t$), then it is easy to show the following: Consider the partitioning defined by the $h^{\text{th}}$ level where $h=1/\epsilon$; such a partitioning would have $\smash{M=2^{2^h}}$ parts. Furthermore, by the uniform convergence assumption, the diameter of each part would be bounded by $\delta \le (1/h) \cdot 2^{-h/2}$. Thus, the total error would be $O(\delta \sqrt{\log M}) = O(\epsilon)$. So in this idealized setting of uniform convergence we would be able to achieve a total error of $\epsilon$ using $2^{2^{O(1/\epsilon)}}$ clusters. 

While this idealized setting is not always achievable, we can still use the majorizing measures theorem to achieve the same asymptotic guarantee. In particular, we use the hierarchical clustering given by the majorizing measures theorem to obtain a partition of $T$ with $2^{2^{O(1/\epsilon)}}$ parts such that the total error of our approximator is $O(\epsilon)$; see~\Cref{claim:chop-termination} and the paragraph preceding it for an explicit description of the clustering procedure.  
This completes an overview of the technical ideas behind our main result, \Cref{thm:GP-sparsification}. 

The Gaussian process guaranteed to exist by~\Cref{thm:GP-sparsification} is non-centered; by simulating each of the constants $c_s$ in~\Cref{thm:GP-sparsification} using auxiliary Gaussian random variables, we show in~\Cref{cor:shifted-shinboku} how to obtain a centered Gaussian process that approximates the supremum of the original Gaussian process. 

%

\paragraph{A Junta Theorem for Norms over Gaussian Space.}
Recall that \Cref{thm:norm-junta-theorem} informally says that every norm over Gaussian space is essentially a junta. 
We obtain~\Cref{thm:norm-junta-theorem} as a consequence of a multiplicative sparsification lemma for the supremum of a Gaussian process on a \emph{symmetric} set (\Cref{cor:multiplicative-approximation}).\footnote{A set $T\sse\R^n$ is \emph{symmetric} if $-t \in T$ whenever $t \in T$.}

Let $\nu(\cdot)$ be any norm over $\R^n$; thanks to homogeneity of norms, we may assume that $\E[\nu(\bg)] = 1$ for $\bg \sim N(0,I_n)$. Since $\nu$ is a norm, there exists a symmetric set $T \sse \R^n$ such that 
\[
	\nu(x) = \sup_{t \in T} x \cdot t. 
\]
In particular, the set $T$ corresponds to the unit ball of the \emph{dual norm} $\nu^\circ$~\cite{TkoczNotes}. 
Consequently, our earlier assumption can be restated as 
\[
	\Ex_{\bg\sim N(0,I_n)}\sbra{\nu(\bg)} = \Ex_{\bg\sim N(0,I_n)}\sbra{\sup_{t \in T} t \cdot \bg} = 1.
\]
Building on~\Cref{thm:GP-sparsification} and~\Cref{cor:shifted-shinboku}, we show the existence of a symmetric set $S \sse \R^n$ with 
\[
	|S| \leq \pbra{\frac{1}{\eps}}^{\exp\pbra{O\pbra{\frac{1}{\eps^3}}}}
\]
such that 
\[
	\Prx\sbra{\abs{\sup_{t\in T} t\cdot \bg - \sup_{s\in S} s\cdot \bg} \leq \eps} \geq 1- \eps.
\] 
An anti-concentration lemma for the supremum of a symmetric Gaussian process (\Cref{lemma:prof-de})---which we prove using the celebrated $S$-inequality of Lata\l{}a and Oleszkiewicz~\cite{Latala1999}---allows us to convert this additive bound into a multiplicative guarantee, yielding~\Cref{cor:multiplicative-approximation}. 
Finally, as the set $S$ is symmetric, the function $\psi:\R^n\to\R$ defined as $\psi(x) = \sup_{s\in S}s\cdot x$ is a norm that depends on at most $|S|$ directions, completing the proof.

\paragraph{Sparsifying Intersections of Narrow Halfspaces.}

For the purposes of this discussion, assume that we have a convex set $K$ which uniformly has ``unit width,'' i.e.
\[
	K = \bigcap_{t \in T} \{x \in \mathbb{R}^n: t \cdot x \le 1\}
\]
where each $t$ is a unit vector in $\mathbb{S}^{n-1}$. Consider the Gaussian process $\{\bX_t\}_{t \in T}$ where $\bX_t :=\bg \cdot t$, and note that 
$K = \{x: f_T(x) \le 1\}$ where 
$f_T(x): = \sup_{t \in T} t \cdot x$. 
Thanks to~\Cref{thm:GP-sparsification}, we know that there is a subset $S \subset T$ (with $|S| = 2^{2^{O(1/\epsilon)}}$) and suitable constants $\{c_s\}_{s \in S}$ such that 
\[
	\mathbf{E}\sbra{\abs{\sup_{t \in T} t \cdot \bg - \sup_{s \in S} \cbra{s \cdot \bg + c_s}}} = \epsilon \cdot \E\sbra{\sup_{t \in T} t \cdot \bg }.
\] 
Thus, a natural way to define an approximator to the convex set $K$ is to consider the polytope $J$, defined as
\[
	J = \cbra{x \in \mathbb{R}^n: \textrm{for all } s\in S, \  s \cdot x + c_s \le 1}. 
\]
The only remaining step is to show that $J$ is  close to $K$ (in the sense that the Gaussian volume of their symmetric difference is small) is to argue that the function $f_T(x)$ is  anti-concentrated, which follows from a result of Chernozhukov et al.~\cite{CCK-2}. 

While in general the set $T$ need not be a subset of $\mathbb{S}^{n-1}$ (i.e.~all the vectors need not have unit norm), it is not too difficult to adapt the above argument to the more general setting; we defer details of this to \Cref{sec:polytope-sparsification}.

\subsection{Discussion} \label{sec:discussion-applications}

We first provide some broader context by recalling similar-in-spirit results in the discrete setting of $\zo^n$, and then discuss some of the qualitative aspects and applications of our Gaussian-space sparsifiers.

\paragraph{\bf Analogy with Sparsification of CNFs and Submodular Functions over $\zo^n$.}
Our sparsification result for convex sets of bounded geometric width (\Cref{def:width}), \Cref{thm:polytope-sparsification}, can be viewed as a Gaussian space analogue of recent results from Boolean function analysis on sparsifying \emph{CNF formulas} of bounded width.  Recall that a width-$w$ CNF formula is a conjunction of an arbitrary number of clauses (Boolean disjunctions), each of which contains at most $w$ literals and hence ``lops off'' at least a $2^{-w}$ fraction of all points in $\zo^n$. This is analogous to how a convex set of geometric width $k$ is an intersection of an arbitrary number of halfspaces in $\R^n$, each at distance at most $k$ from the origin and hence ``lopping off'' at least an $\approx \exp(-k^2/2)$ fraction of Gaussian space under $N(0,I_n)$. 

In the Boolean context, the state-of-the art result of Lovett et al.~\cite{lovett2021decision} gives, for any initial width-$w$ CNF over $\zo^n$, an $\eps$-approximating width-$w$ CNF consisting of $s=(2+ w^{-1} \log(1/\eps))^{O(w)}$ clauses---note that this size bound is completely independent of the dimension $n$ and the number of clauses in the initial CNF formula.
Our \Cref{thm:polytope-sparsification} similarly gives a sparsifying polytope using a number of halfspaces that depends only on the error parameter $\eps$ and the geometric width of the original convex set, with no dependence on the ambient dimension $n$ or the number of halfspaces in the original convex set.
(Indeed, obtaining such a Gaussian-space analogue of the Boolean CNF sparsification result was the initial impetus for this work.)

We remark that our norm sparsification result, \Cref{thm:norm-junta-theorem}, can also be viewed as analogous to known results \cite{FKV13,FV16} giving junta approximators for \emph{submodular functions} over $\zo^n$.
(Recall that submodular functions are often viewed as being a discrete analogue of convex functions \cite{FKV13,Lovasz82} and that every norm over $\R^n$ is a convex function.)

\paragraph{Sparsification Lower Bounds.}  
In light of our positive results for sparsification, it is natural to ask about lower bounds: How close to optimal are the quantitative upper bounds that we establish?  As mentioned above, there is an upper bound of $(2+ w^{-1} \log(1/\eps))^{O(w)}$ clauses for the analogous Boolean question of sparsifying width-$w$ CNFs over $\zo^n$, and this is known to be best possible up to the hidden constant in the $O(w)$ exponent  \cite{lovett2021decision}.  So for constant width, in the Boolean setting the correct dependence on the error parameter $\eps$ is \emph{poly-logarithmic}, while our upper bound given in \Cref{thm:polytope-sparsification} is \emph{doubly exponential}---a difference of three exponentials!  This contrast naturally motivates the question of whether stronger quantitative results can be obtained in our setting.  

In \Cref{sec:lower-main} we show that the quantitative parameters of our sparsifiers are not too far from the best possible:  roughly speaking, we give a $2^{(1/\eps)^{\Omega(1)}}$ lower bound, which is only one exponential away from our doubly-exponential upper bound. In particular, our lower bound shows that sparsifying constant-geometric-width polytopes over $\R^n$ inherently requires a much worse $\eps$-dependence than sparsification of constant-width CNF formulas over $\zo^n$: the former requires $2^{\Omega(1/\eps)}$  halfspaces, while the latter can be achieved using only $(\log(1/\eps))^{O(1)}$ clauses. 

\paragraph{On Centered and Non-Centered Sparsifiers.}  
Recalling \Cref{eq:glory}, since $S \subseteq T$ the sparsified Gaussian process $\{\bX_s + c_s\}_{s \in S}$ can be viewed as a non-centered (because of the $c_s$'s) sub-process of the original centered Gaussian process $\{\bX_t\}_{t \in T}$.  In \Cref{cor:shifted-shinboku} we show that our main sparsification result, \Cref{thm:GP-sparsification}, can be easily modified to give a sparsifier $S$ of essentially the same size which is a centered process; however, this sparsifier $S$ is not a sub-process of the original Gaussian process (i.e.~$S$ is no longer a subset of $T$). This tradeoff is unavoidable:  In \Cref{sec:proper-impossible} we give a simple example showing that any sparsifier whose size is independent of the ambient dimension cannot both be a sub-process and be centered (and as we explain in \Cref{sec:proper-impossible}, analogous tradeoffs are similarly unavoidable for \Cref{thm:norm-junta-theorem} and \Cref{thm:polytope-sparsification}).

\subsection{Related Work}
\label{subsec:rel}

Our work connects to a broad literature on Gaussian processes and dimensionality reduction. 
Below, we summarize the most relevant results and situate our sparsification theorems within this context. 

\paragraph{The Majorizing Measures Theorem.} 

As described in \Cref{subsec:techniques} (and formally stated in \Cref{subsec:gaussian-process-prelims}), Talagrand's majorizing measures theorem gives a  geometric characterization of the expected supremum $\E[\sup_{t\in T} \bX_t]$ of a Gaussian process, up to constant factors, in terms of a hierarchical decomposition of the index set $T$. 
Our main result, \Cref{thm:GP-sparsification},  goes further: by leveraging the same hierarchical structure underlying Talagrand's result, we show that the supremum itself can be closely approximated (in $L^1$) by the supremum of a sparse, non-centered Gaussian process that is indexed by a subset $S \sse T$ of constant size. 

In this sense, our sparsification theorem may be viewed as a refinement of the majorizing-measure framework: it approximates the random variable $\sup_{t\in T} \bX_t$ itself, rather than merely its expectation, through a sparse geometric representation derived from Talagrand's multi-scale clustering. 
We emphasize, however, that our result does not imply Talagrand's theorem; rather, it uses Talagrand's multi-scale clustering to obtain a stronger guarantee. 

\paragraph{Dimensionality Reduction.} 

At a high level, \Cref{thm:GP-sparsification} shows that when $\E[\sup_{t\in T}\bX_t]$ is small, the corresponding Gaussian process admits a sparse representation of its supremum. 
This theme, that low (geometric or probabilistic) ``complexity'' implies low-dimensional or sparse structure, also underlies several classical results in geometric functional analysis. 
A particularly relevant example is Gordon's theorem~\cite{gordon1988milman}, which can be seen as a close cousin of the Johnson--Lindenstrauss (JL) lemma~\cite{johnson1984extensions}.  
Informally, the JL lemma says that a set of $s$ points in $\R^n$ can be embedded into $\R^{O(\log s)}$ while approximately preserving Euclidean distances. 
Gordon's theorem, by contrast, applies to \emph{arbitrary} (possibly infinite) subsets $T \sse \S^{n-1}$: it states that if $\E[\sup_{t\in T} \bX_t] \leq k$, then $T$ can be embedded into $\R^{O(k^2)}$ while approximately preserving all pairwise distances. 

Both the JL lemma and Gordon's theorem capture the idea that $\E[\sup_{t\in T} \bX_t]$ serves as a quantitative measure of the intrinsic complexity of a set. 
In this sense, our sparsification theorem is conceptually related: it demonstrates that when a Gaussian process has small expected supremum, its supremum can be well approximated by the supremum of a sparse, low-complexity process.
However, our result is somewhat different in both its setting and guarantee. 
First, unlike Gordon's theorem, our index set $T$ is not required to be a subset of $\S^{n-1}$. 
Second, whereas results such as the JL lemma and Gordon's theorem provide geometric preservation guarantees (for example, of distances, norms, or inner products) under random projection, our theorem establishes a form of \emph{probabilistic preservation}: it ensures that the supremum of the original process is well approximated, as a random variable, by that of a sparse sub-process. 
Indeed, one can easily construct examples showing that the former type of guarantee does not imply the latter. 

\paragraph{One-Sided Approximations via Majorizing Measures.}

Lemma~2.3 of Klartag and Mendelson~\cite{klartag2005empirical} is perhaps the most closely related prior result to our work. 
It provides a one-sided approximation to the supremum: specifically, it constructs a subset $T' \sse T$ such that, with high probability,
\[
	\sup_{t\in T} \bX_t \leq \sup_{t' \in T'} \bX_{t'} + \eps\cdot\E\sbra{\sup_{t\in T} \bX_t}\,,
\]
and moreover $|T'| \leq 4^{\frac{1}{\eps^2}}$. 
Their proof also relies on the majorizing measures theorem, but effectively uses only a \emph{single level} of the hierarchical clustering.
Consequently, the resulting approximation cannot achieve an $L^1$-guarantee of the kind we establish in \Cref{thm:GP-sparsification}. 
Indeed, the use of multi-scale clusterings---precisely what underlies the generic chaining and our sparsification theorem---is well known to be essential even for bounding $\E[\sup_{t\in T}\bX_t]$ up to constant multiplicative factors. 
Intuitively, when the high-probability event in their lemma fails (with probability roughly $e^{-O(1/\eps^2)}$), the supremum can deviate significantly, preventing control of the expected $L^1$ error.


\section{Preliminaries}
\label{sec:prelims}

We use boldfaced letters such as $\bx, \bX$, etc.~to denote random variables (which may be real- or vector-valued; the intended type will be clear from the context).
We write $\bx \sim \calD$ to indicate that the random variable $\bx$ is distributed according to probability distribution $\calD$.

We will frequently identify a set $K\sse\R^n$ with its $0/1$-valued indicator function. Given $t \in R$ and a set $K\sse\R^n$, we write  $tK$ for the $t$-dilation of $K$, i.e. $tK := \{tx : x\in K\}$. 
Unless explicitly stated otherwise, $\|\cdot\|$ denotes the $\ell_2$-norm.
We write $\S^{n-1}$ to denote the unit $\ell_2$-sphere in $\R^n$, i.e. 
\[
	\S^{n-1} := \cbra{x\in\R^n : \|x\| = 1}.
\]
Throughout, $\{e_i\}_{i\in[n]}$ will denote the collection of standard basis vectors in $\R^n$.
We will write $\log$ to denote logarithm base two and write $\ln$ to denote the natural logarithm $\log_e$. Finally, $[n] := \{1,\ldots, n\}$. 

\subsection{Gaussian Random Variables}
\label{subsec:gaussian-tail-bounds}

Identifying $0\equiv 0^n$ and writing $I_n$ for the $n\times n$ identity matrix, $N(0, I_n)$ will denote the $n$-dimensional {standard} Gaussian distribution. We write $\vol(K)$ to denote the Gaussian measure of a (Lebesgue measurable) set $K \subseteq \R^n$, i.e.  
\[\vol(K) := \Prx_{\bg \sim N(0,I_n)}[\bg \in K].\]  

We recall the following standard tail bound on Gaussian random variables:

\begin{proposition}[Proposition~2.1.2 of \cite{vershynin2018high}] \label{prop:gaussian-tails}
	Suppose $\bg\sim N(0,1)$ is a one-dimensional Gaussian random variable. Then {for all $r>0$,}
\[
\left({\frac 1 r} - {\frac 1 {r^3}} \right) \cdot \varphi(r)
\leq \Prx_{\bg \sim N(0,1)}[\bg \geq r] \leq
{\frac 1 r} \cdot \varphi(r),
\]
where $\phi$ is the one-dimensional standard Gaussian density, which is given by 
\[\phi(x) := \frac{1}{\sqrt{2\pi}}e^{-x^2/2}.\]
\end{proposition}

We will use the following well-known bounds on the maximum of i.i.d.~Gaussians:

\begin{proposition} \label{prop:gaussian-sup-bounds}
	Suppose $\bg \sim N(0,I_n)$. Then 
	\begin{equation} \label{eq:sup-of-gaussians}
		\Ex_{\bg\sim N(0,I_n)}\sbra{\max_{i\in [n]} \bg_i} = \Theta\pbra{\sqrt{\ln n}}
		\qquad\text{and}\qquad 
		\Varx_{\bg\sim N(0,I_n)}\sbra{\max_{i\in [n]} \bg_i} = \Theta\pbra{\frac{1}{\ln n}}.
	\end{equation}
	In particular for $n$ large enough, 
	\begin{equation} \label{eq:fine-mean-sup-gaussians}
		\Ex_{\bg\sim N(0,I_n)}\sbra{\max_{i\in [n]} \bg_i} \in \sbra{\sqrt{2\ln n}\pbra{1 - \frac{4}{\ln n}}, \sqrt{2\ln n}\pbra{1 + \frac{4}{\ln n}}}.
	\end{equation}
	Furthermore, we have 
	\begin{equation} \label{eq:sup-of-abs-gaussians}
		\Ex_{\bg\sim N(0,I_n)}\sbra{\max_{i\in [n]} |\bg_i|} = \Theta\pbra{\sqrt{\ln n}}
		\qquad\text{and}\qquad 
		\Varx_{\bg\sim N(0,I_n)}\sbra{\max_{i\in [n]} |\bg_i|} = O\pbra{\frac{1}{\ln n}}.
	\end{equation}
\end{proposition}

See Appendix~A.2 of~\cite{chatterjee2014superconcentration} for the bound on expectation in~\Cref{eq:sup-of-gaussians} and Chapter~5.2 of~\cite{chatterjee2014superconcentration} for the bound on the variance. 
See Example~10.5.3 of~\cite{david2004order} (or alternatively \cite{STACKEXCHANGE}) for \Cref{eq:fine-mean-sup-gaussians}. 
See Exercise~2.11 of~\cite{wainwright2019high} for the bound on the expectation in~\Cref{eq:sup-of-abs-gaussians}; the bound on the variance follows from Chapter~5.2 of~\cite{chatterjee2014superconcentration}.

For $L > 0$, recall that a function $f:\R^n\to\R$ is \emph{$L$-Lipschitz} if $|f(x) - f(y)| \leq L\cdot\|x-y\|$. 
We will require the following concentration inequality for Lipschitz functions of Gaussian random variables:

\begin{proposition}[Theorem~5.2.2 of~\cite{vershynin2018high}]
\label{prop:lipschitz-concentration}
	Suppose $f:\R^n\to\R$ is an $L$-Lipschitz function. Then 
	\[
		\Prx_{\bg\sim N(0,I_n)}\sbra{\abs{f(\bg) - \Ex[f]} \geq t} \leq 2\exp\pbra{-\frac{t^2}{2L^2}}.
	\]
\end{proposition}

\subsection{Gaussian Processes and Suprema}
\label{subsec:gaussian-process-prelims}

Recall that a \emph{random process} is a collection of random variables $\{\bX_t\}_{t\in T}$ on the same probability space which are indexed by elements $t$ of some set $T$.

\begin{definition}[Gaussian process] \label{def:gaussian-process}
	A random process $\{\bX_t\}_{t\in T}$ is called a \emph{Gaussian process} if, for any finite subset $T_0 \sse T$, the random vector $\{\bX_t\}_{t\in T_0}$ has a multivariate Gaussian distribution. We say $\{\bX_t\}_{t\in T}$ is \emph{centered} if $\E[\bX_t] = 0$ for all $t\in T$.
\end{definition}

We assume throughout that all Gaussian processes are \emph{separable};  
we will not dwell on this technical point and refer the reader to Section~2.2 of~\cite{LedouxTalagrand} for a detailed discussion.\ignore{Separability allows us to handle countably-infinite dimensional Gaussian processes $\sum_{t>0} a_t g_t$ by taking the limit of finite dimensional Gaussian processes.}  

\begin{definition}[Canonical Gaussian process] \label{eg:canonical-gaussian-process}
	Given $T\sse\R^n$, the centered random process $\{\bX_t\}_{t\in T}$ defined as
	\[
		\bX_t := \bg\cdot t \quad \text{where}~\bg\sim N(0,I_n)
	\]
	is a Gaussian process that we refer to as the \emph{canonical Gaussian process on $T$}.
\end{definition}

It is a standard fact that any centered Gaussian process can be realized as a canonical process (see Section~7.1.2~of~\cite{vershynin2018high}) and consequently we will focus on canonical Gaussian processes throughout this paper. 
The following notation will be convenient: 

\begin{definition}
Given a set $T\sse\R^n$, we define
\[
	f_T(x) := \sup_{t\in T} x\cdot t. 
\]
\end{definition}

We remark that if $T$ is a closed convex set then $f_T$ is known as the ``support function'' of $T$ (though we will be interested in arbitrary bounded subsets $T$ of $\R^n,$ which need not be convex sets). 
Note that when $\bg\sim N(0,I_n)$, $f_T(\bg) = \sup_{t\in T} \bX_t$ 
where $\{\bX_t\}_{t\in T}$ is the canonical Gaussian process on $T$.
We will frequently rely on the fact that the supremum of a Gaussian process has subgaussian tails, which we recall below. 

\begin{proposition}[Appendix~A.5 of~\cite{chatterjee2014superconcentration}] \label{fact:sup-has-subgaussian-tails}
	Let $\{\bX_t\}_{t\in T}$ be a centered Gaussian process indexed by a set $T\sse\R^n$ and let $r \geq 0$. Then 
	\[\Prx\sbra{\abs{\sup_{t\in T} \bX_t - \Ex\sbra{\sup_{t\in T}\bX_t}} \geq r} \leq 2\exp\pbra{\frac{-r^2}{2\sup_{t\in T}\Var\sbra{\bX_t}}}.\]
\end{proposition}

For a canonical Gaussian process $\{\bX_t\}_{t\in T}$ as in \Cref{fact:sup-has-subgaussian-tails}, we have that $\sup_{t\in T}\Var\sbra{\bX_t} = \sup_{t \in T}\|t\|^2$, and hence $\sup_{t\in T}\Var\sbra{\bX_t} \geq \diam(T)^2/4$.
Finally, the magnitude of the canonical Gaussian process on $T$ is captured by the following important geometric quantity:

\begin{definition}[Gaussian width] \label{def:gaussian-width}
	The \emph{Gaussian width} of a set $T \sse \R^n$ is defined as 
	\[
		w(T) := \Ex_{\bg\sim N(0,I_n)}\sbra{f_T(\bg)} = \Ex_{\bg\sim N(0,I_n)}\sbra{\sup_{t\in T} \bg\cdot t}.
	\]
\end{definition}

Gaussian width was originally introduced in geometric functional analysis and asymptotic convex geometry~\cite{aga-book,artstein2021asymptotic}. It is easy to see that $w(T)$ is finite if and only if $T$ is bounded; we refer the reader to Chapter~7 of \cite{vershynin2018high} for further background on Gaussian width.

\subsubsection{Anti-Concentration Bounds}
\label{subsec:prelims-anticonc}

We will require the following anti-concentration bound on $f_T(\bg)$ obtained by~\cite{CCK-2} (which is an easy consequence of an analogous statement for Gaussian random vectors obtained in~\cite{CCK15}):   

\begin{theorem}[Theorem~2.1 of~\cite{CCK-2}] \label{thm:CCK}
	Suppose $T \sse \S^{n-1}$. Let $\{\bX_t\}_{t\in T}$ be the canonical Gaussian process on $T$. For every $\theta \in \R$ and $\eps > 0$, we have 
	\begin{equation} \label{eq:cck}
		\Prx_{\bg\sim N(0,I_n)}\sbra{\abs{f_T(\bg) - \theta} \leq \eps} \leq 4\eps\pbra{1 + w(T)}
	\end{equation}
	where $w(T)$ is as in~\Cref{def:gaussian-width}.
\end{theorem}

\subsubsection{Talagrand's Majorizing Measures Theorem}
\label{subsec:talagrand-tree}

We refer the reader to~\cite{rvh-notes,vershynin2018high,talagrand2022upper} for further background on the majorizing measures theorem as well as applications thereof. 

\begin{definition}[Admissible sequence] \label{def:admissible-sequence}
	Let $T\sse\R^n$. An \emph{admissible sequence $\{\calA_h\}_{h \in \N}$ of partitions of~$T$} is defined to be a collection of partitions with the following properties: 
	\begin{enumerate}
		\item[(i)] $\calA_0 = \{T\}$,
		\item[(ii)] $\calA_{h+1}$ is a refinement of $\calA_{h}$ for all $h$, and 
		\item[(iii)] $\abs{\calA_h} \leq 2^{2^h}$ for all $h \in \N$.
	\end{enumerate}
\end{definition}

\begin{definition}[Talagrand's $\gamma_2$ functional] \label{def:gamma-2-functional}
	Let $T\sse\R^n$. Given $t \in T$ and a partition $\calP$ of $T$, we write $\calP(t)$ for the part of~$\calP$ containing~$t$. We define 
	\[
		\gamma_2(T) = \inf_{\substack{\text{admissible}\,\calA}} 
    ~\sup_{t \in T} \sum_{h \geq 0} 2^{h/2}\cdot  \diam(\calA_h(t))
	\]
	where $\diam(S)$ denotes the diameter of the set $S$, i.e. $\diam(S) = \sup_{s_1, s_2 \in S} \|s_1 - s_2\|_2$.
\end{definition}

The \emph{majorizing measures theorem}, with upper bound due to Fernique~\cite{Fernique75} and lower bound due to Talagrand~\cite{og-mm}, is the following:

\begin{theorem}[Theorem~2.10.1 of~\cite{talagrand2022upper}] \label{thm:talagrand-mm}
	Let $T\sse\R^n$ and let $\cbra{\bX_t}_{t\in T}$ be a centered Gaussian process. Then there exists a universal constant $L$ such that 
	\[
		\frac{1}{L}\cdot \gamma_2(T) \leq \Ex\sbra{\sup_{t\in T}\bX_t} \leq L\cdot \gamma_2(T).
	\]
\end{theorem}

\subsection{Polyhedral Approximation Under the Gaussian Distance}
\label{subsec:polytope-approximation-prelims}

We will be interested in approximating subsets of $\R^n$ under the Gaussian distance:

\begin{definition}
	Given two measurable sets $K, L \sse\R^n$, we define the \emph{Gaussian distance between $K$ and $L$} to be
	\[
		\dG(K,L) := \Prx_{\bx\sim N(0,I_n)}\sbra{K(\bx) \neq L(\bx).}
	\]
	In other words, $\dG(K,L) = \vol(K\, \triangle\, L)$, i.e. the Gaussian measure of the symmetric difference of the sets $K$ and $L$. 
\end{definition}

Recall that every convex set can be written as an intersection of (possibly infinitely many) halfspaces. 
The following upper bound on the number of halfspaces needed to approximate a polytope under the Gaussian distance was obtained in~\cite{DNS23-polytope}: 

\begin{theorem}[Theorem~19 of~\cite{DNS23-polytope}] 
\label{thm:polytope-approximation-ub}
	Given a convex set $K \sse \R^n$ and $\eps \in (0, 10^{-3})$, there exists a set $L \sse \R^n$ which is an intersection of $(n/\eps)^{O(n)}$ halfspaces such that $\dG(K, L) \leq \eps$.
\end{theorem}

\cite{DNS23-polytope} also gave a mildly-exponential lower bound on the number of halfspaces needed to approximate the $\ell_2$-ball of radius $\sqrt{n}$ to constant accuracy. The arguments of \cite{DNS23-polytope} give the following lower bound:

\begin{theorem}[Theorem~63 of~\cite{DNS23-polytope}] 
\label{thm:ball-approximation-lb}
%
	Let $B \sse \R^n$ be an origin-centered $\ell_2$ ball with radius $r \in [\sqrt{n} - 1, \sqrt{n}+1]$, so $\Vol(B) \in [\tau,1-\tau]$ for an absolute constant $\tau>0$.
	There exists an absolute constant $\kappa \in (0, 0.5)$ such that any intersection of halfspaces $L \sse \R^n$ with $\dG(B, L) \leq \kappa$ must have at least $2^{\Omega(\sqrt{n})}~\text{halfspaces}$.
\end{theorem}

We note that the lower bound in~\Cref{thm:ball-approximation-lb} is tight up to constant factors hidden by the $\Omega(\cdot)$-notation; we refer the reader to Section~5 of~\cite{DNS23-polytope} for a probabilistic construction of an intersection of $2^{O(\sqrt{n})}$ halfspaces that approximates the $\ell_2$-ball of radius $\sqrt{n}$ to constant accuracy. 

We will be interested in convex sets that have bounded \emph{geometric width}, which we define as follows:

\begin{definition} [Geometric width of a convex set] \label{def:width}
	A convex set $K \subseteq \R^n$ is said to have \emph{geometric width at most $r$} if $K$ can be expressed as an intersection of (possibly infinitely many) halfspaces, each of which is of the form $\Indicator\cbra{u \cdot x \leq r'}$ where $\|u\|=1$ and $r'\leq r$.
\end{definition}

Thus the geometric width of a single halfspace $H=\Indicator\cbra{u \cdot x \leq r}$ where $\|u\|=1$ is $r$, and the geometric width of an origin-centered $\ell_2$-ball of radius $r$ is $r$.


\section{Sparsifying Gaussian Processes with Bounded Expected Suprema}
\label{sec:shinboku}

In this section, which is at the heart of our technical contribution, we show that if a set $T$ is bounded (equivalently, has bounded Gaussian width) then the supremum of the canonical Gaussian processes on $T$ can be approximated by the supremum of a sparse or ``low-dimensional'' (non-centered) Gaussian process. 

\begin{theorem}[Sparsifying suprema of Gaussian processes] \label{prop:shinboku}
	Let $\eps \in (0, 0.5)$. For any $T\sse\R^n$ with $0 < w(T) < \infty$, there exists a set $S \sse T$ with 
	\[
		|S| = 2^{2^{O\pbra{\frac{w(T)}{\epsilon}}}}
	\]
	and constants $\{c_s\}_{s\in S}$ satisfying $0 \leq c_s \leq w(T)$ such that 
	\[
		\Ex_{\bg\sim N(0,I_n)}\sbra{\abs{f_T(\bg) - \sup_{s\in S} \, \{\bg\cdot s + c_s\}}} \leq \eps.
	\]
\end{theorem}

Note that \Cref{thm:GP-sparsification} is easily obtained from \Cref{prop:shinboku} by rescaling the Gaussian process $\{\bX_t)\}_{t \in T}$ by a factor of $w(T)$ so that the resulting process has Gaussian width 1. 
Towards~\Cref{prop:shinboku}, we require the following technical lemma:

\begin{lemma} \label{lemma:cut-union-bound}
	Let $\calP$ be a partition of $T$ into finitely many pieces. Define $S\sse\T$ as
	\[
		S := \{s_P : P \in \calP\}
	\]
	where each $s_P$ is an arbitrary (but fixed) representative element of the part $P$. Define 
	\[
		c_P := \Ex_{\bg\sim N(0,I_n)}\sbra{\sup_{t \in P}\, \bg\cdot(t - s_P)}.
	\]
	Then for any $\delta \geq 0.5\cdot\max_{P \in \calP}\diam(P)$ we have 
	\[
		\Ex_{\bg\sim N(0,I_n)}\sbra{\abs{f_T(\bg) - \sup_{P \in \calP}\cbra{\bg\cdot s_P + c_P}}} \leq \delta\pbra{1 + \sum_{P\in\calP} \exp\pbra{\frac{-2\delta^2}{\diam(P)^2}}}. 
	\]
\end{lemma}

\begin{proof}
	For any $g\in\R^n$, we have
	\begin{align*}
		\abs{f_T(g) - \sup_{P\in \calP}\cbra{g\cdot s_P + c_P}} &=
		\abs{\sup_{t\in T} g\cdot t - \sup_{P \in \calP}\cbra{g\cdot s_P + c_P}}\\
		&= \abs{\sup_{P \in \calP} \sup_{t \in P} g\cdot t - \sup_{P\in\calP}\cbra{g\cdot s_P + c_P}}\\
		&\leq \sup_{P\in\calP}\abs{\sup_{t\in P} g\cdot t - \pbra{g\cdot s_P + c_P}} \\
		&= \sup_{P\in\calP} \Bigg|{\sup_{t \in P} \cbra{g\cdot (t - s_P) - c_P}} \Bigg|,
	\end{align*}
	where the inequality is using convexity of the sup function and Jensen's inequality.
	Motivated by this, for each $P \in \calP$ let $\{\bY^{(P)}_t\}_{t\in P}$ be the following centered Gaussian process:
	\[\bY^{(P)}_t :=  \bg\cdot (t - s_P) \qquad\text{where}~\bg\sim N(0,I_n).\]
	Note that 
	\[
		\Ex\sbra{\sup_{t\in P}\bY_t^{(P)}} = c_P.
	\]
	We thus have
	\begin{equation} \label{eq:berkeley}
		\Ex_{\bg\sim N(0,I_n)}\sbra{\abs{f_T(\bg) - \sup_{P\in \calP}\cbra{\bg\cdot s_P + c_P}}} 
		\leq 
		\Ex_{}\sbra{\sup_{P\in\calP}\abs{\sup_{t \in P} \bY_t^{(P)} - \E\sbra{\sup_{t \in P} \bY_t^{(P)}}}}.
	\end{equation}
	In particular, we can write 
	\begin{align}
		\text{R.H.S. of~\Cref{eq:berkeley}} &= \int_{r=0}^\infty \Prx\sbra{\sup_{P\in\calP}\abs{\sup_{t \in P} \bY_t^{(P)} - \E\sbra{\sup_{t \in P} \bY_t^{(P)}}} \geq r} \, dr \nonumber \\[0.5em] 
		&\leq \delta + \int_{r=\delta}^\infty \Prx\sbra{\sup_{P\in\calP}\abs{\sup_{t \in P} \bY_t^{(P)} - \E\sbra{\sup_{t \in P} \bY_t^{(P)}}} \geq r} \, dr \nonumber \\[0.5em] 
		&\leq \delta + \int_{r=\delta}^\infty \sum_{P\in\calP} \Prx\sbra{\abs{\sup_{t \in P} \bY_t^{(P)} - \E\sbra{\sup_{t \in P} \bY_t^{(P)}}} \geq r} \, dr \label{eq:union-bound} \\[0.5em] 
		&= \delta + \sum_{P\in\calP} \int_{r=\delta}^\infty \Prx\sbra{\abs{\sup_{t \in P} \bY_t^{(P)} - \E\sbra{\sup_{t \in P} \bY_t^{(P)}}} \geq r} \, dr \label{eq:fubb} \\[0.5em]
		&\leq \delta + \sum_{P\in\calP} \int_{r=\delta}^\infty 2\exp\pbra{\frac{-2r^2}{ \diam(P)^2}} \, dr \label{eq:rvh-subgaussian-soup}
	\end{align}
	where \Cref{eq:union-bound} follows via a union bound, \Cref{eq:fubb} uses Fubini's theorem, and \Cref{eq:rvh-subgaussian-soup} follows from~\Cref{fact:sup-has-subgaussian-tails}.
	
	A change of variables together with the Gaussian tail bound from \Cref{prop:gaussian-tails} gives 
	\begin{align*}
		\int_{r=\delta}^\infty \exp\pbra{\frac{-2r^2}{\diam(P)^2}} \,dr
		&= \sqrt{\frac{\pi}{2}}\cdot \diam(P)\cdot \Prx\sbra{N(0,1) \geq \frac{2\delta}{\diam(P)}}\\
		&\leq \frac{\diam(P)^2}{4\delta}\cdot\exp\pbra{\frac{-2\delta^2}{\diam(P)^2}}.
	\end{align*}
	By the assumption that $2 \delta \geq \diam(P)$ for all $P\in\calP$, combining this with~\Cref{eq:rvh-subgaussian-soup} gives
	\[\text{R.H.S. of~\Cref{eq:berkeley}} \leq \delta\pbra{1 + \sum_{P\in\calP} \exp\pbra{\frac{-2\delta^2}{\diam(P)^2}}}\]
	which completes the proof.
\end{proof}

With~\Cref{lemma:cut-union-bound} in hand, \Cref{prop:shinboku} follows via the lower bound from Talagrand's majorizing measures theorem (\Cref{thm:talagrand-mm}):

\begin{proof}[Proof of~\Cref{prop:shinboku}]
	Let $\{\bX_t\}_{t\in T}$ be the canonical Gaussian process on $T$. By~\Cref{thm:talagrand-mm}, we have that 
	\[
		\frac{1}{L}\cdot\gamma_2(T) \leq w(T)
	\]
	for some universal constant $L$. In particular, there exists an admissible sequence (cf. \Cref{def:admissible-sequence}) $\calA$ such that 
	\begin{equation} \label{ineq:lhsmm}
		\sum_{h \geq 0} 2^{h/2}\cdot\diam(\calA_h(t)) \leq L\cdot w(T) \quad \text{for all}~t \in T.
	\end{equation}
	We will construct an appropriate partition of~$T$ from the admissible sequence $\calA$ and then appeal to~\Cref{lemma:cut-union-bound} to complete the proof. 
	Let $\delta$ be a parameter we will set shortly. 
	Construct $\calP$  as follows: In stage $h = 1, 2, 3, \dots$, we place into~$\calP$ all parts~$P$ from $\calA_h$ satisfying \begin{equation} \label{eq:part-condition}
	2\cdot\diam(P) \leq \delta\cdot 2^{-h/2},
	\end{equation}
	provided $P \not \subseteq \bigcup \calP$ already.
	
	\begin{claim} \label{claim:chop-termination}
	    This process terminates with a partition~$\calP$ of~$T$ by stage $h = 1 + \floor{2L\cdot w(T) / \delta}$.
	\end{claim}
	\begin{proof}
	    Suppose otherwise. Then there exists some $t \in T$ with 
	    \[
	    	2^{h/2}\cdot\diam(\calA_h(t)) > \frac{\delta}{2} \quad\text{for all}~1 \leq h \leq 1 + \floor{2L\cdot w(T) / \delta},
	    \]
	    which contradicts \Cref{ineq:lhsmm}.
	\end{proof}	
	
	As an immediate consequence of~\Cref{claim:chop-termination} and Item~(iii) of \Cref{def:admissible-sequence} it follows that 
\begin{equation} \label{eq:partition-size}
		|\calP| \leq 2^{2^{1 + \floor{2L\cdot w(T)/\delta}}}.
	\end{equation}
	Next, note that at stage $h$ of the above process, we add at most $2^{2^h}$ parts to $\calP$ and each part $P$ satisfies \Cref{eq:part-condition}. Consequently, 
	\begin{align}
\sum_{P\in\calP} \exp\pbra{\frac{-2\delta^2}{\diam(P)^2}}
		\leq \sum_{P\in\calP} \exp\pbra{\frac{-\delta^2}{2\cdot\diam(P)^2}} &\leq \sum_{h=1}^\infty 2^{2^{h}}\cdot \exp\pbra{-2^{h+1}} \nonumber \\
		&= \sum_{h=1}^\infty \exp\pbra{\ln 2\cdot 2^h - 2\cdot 2^h} \nonumber \\ 
		&\leq \sum_{h=1}^\infty \exp\pbra{-2^{h}} 
		\leq 1. \label{eq:main-lemma-geometric-series}
	\end{align}
Set $\delta := \eps/2$; it thus follows from~\Cref{lemma:cut-union-bound} that for appropriate sequences of vectors $\{s_P\}_{P\in\calP}$ and constants $\{c_P\}_{P\in\calP}$,
	\begin{align*}
		\Ex_{\bg\sim N(0,I_n)}\sbra{\abs{f_T(\bg) - \sup_{P \in \calP}\cbra{\bg\cdot s_P + c_P}}} 
		&\leq \delta\pbra{1 + \sum_{P\in\calP} \exp\pbra{\frac{-2\delta^2}{\diam(P)^2}}} \\
		&\leq \epsilon
	\end{align*}
	thanks to \Cref{eq:main-lemma-geometric-series} and our choice of $\delta$. (Note that \Cref{eq:part-condition} immediately implies that $\delta \geq \sqrt{2}\cdot\diam(P)$ for all $P\in\calP$, as required by~\Cref{lemma:cut-union-bound}.) Furthermore, taking $S = \{s_P\}_{P\in\calP}$, we have by~\Cref{eq:partition-size} that 
	\[
		|S| = |\calP| 
		\leq 2^{2^{O\pbra{\frac{w(T)}{\epsilon}}}}
	\]
	as desired, concluding the proof of \Cref{prop:shinboku}.
\end{proof}

\subsection{Sparsifying Using a Centered Gaussian Process}
\label{subsec:discussion-shinboku}

Observe that the sparsified process $\{\bX_s + c_s\}_{s \in S}$ guaranteed by~\Cref{prop:shinboku} is a non-centered Gaussian process. 
It is not too difficult to obtain a sparsifier of $\{\bX_t\}_{t \in T}$ correponding to a \emph{centered} Gaussian process, although the vectors used by this sparsifier no longer lie in the original set $T$:

\begin{corollary} [Sparsifying using a centered Gaussian process]
\label{cor:shifted-shinboku}
	Let $\eps \in (0, 0.5)$. 
	For any $T\sse\R^n$ with $w(T) < \infty$, there exists a set $S  \subset \R^{n+A}$ with 
		\[
			|S| = 2^{2^{O\pbra{\frac{w(T)}{\epsilon}}}} 
			\qquad\text{and}\qquad 
			A = \exp\pbra{\Theta\pbra{\frac{w(T)}{\epsilon}}}
		\]
	such that (concatenating each vector $t \in T$ with the $A$-dimensional vector $0^A$, which lets us view $f_T$ as  $f_T:\R^{n+A} \to \R$) 	
		\begin{equation} \label{eq:shifted-shinboku-goal}
			\Ex_{\bg\sim N(0,I_{n+A})}\sbra{\abs{f_T(\bg) - f_S(\bg)}} \leq \eps.
		\end{equation}
	Moreover, each $s \in S$ is of the form $s = (s', \alpha e_i)$ for some $i \in [A]$ and $s' \in T$. Furthermore $(s', -\alpha e_i)$ belongs to $S$ whenever $(s', \alpha e_i)$ belongs to $S$.
\end{corollary}

Looking ahead, we will rely on the symmetry of coordinates in $A$ in the following subsection (\Cref{subsec:multiplicative-approximation}).

\begin{proof}
From \Cref{prop:shinboku}, for the function $f_T(\cdot)$, there is a set $S'$ and a suitable set of constants $\{c_{s'}\}_{s' \in S'}$ with $c_{s'} \geq 0$ 
such that 
\begin{equation}~\label{eq:guarantee-from-main-theorem}
\Ex_{\bg\sim N(0,I_{n})}\sbra{\abs{f_T(\bg) - \sup_{s'\in S'} \, \{\bg\cdot s'+c_{s'}\}}} \leq \frac{\eps}{2},
\end{equation} 
where $|S'| = 2^{2^{O\pbra{\frac{w(T)}{\epsilon}}}}$. 
The idea is to now replace each constant $c_{s'}$ by a suitable maximum of independent Gaussian random variables. In particular, let $A$ be a sufficiently large number that we will set shortly and let 
\[
	\mu_A := \Ex_{\bg\sim N(0,I_A)}\sbra{\sup_{i\in[A]}|\bg_i|},
	\qquad\qquad
	\sigma^2_A := \Varx_{\bg\sim N(0,I_A)}\sbra{\sup_{i\in[A]} |\bg_i|}.
\]
For every $s' \in S'$, define the set $\Aux(s') \sse \R^{n+A}$ as follows: 
\[
	\Aux(s') := \bigcup_{j\in[A]} \cbra{\bigg( s',  \frac{c_{s'} e_j}{\mu_A} \bigg), \bigg( s',  \frac{-c_{s'} e_j}{\mu_A} \bigg)}.
\]
Finally, we define the set $S \sse \R^{n+A}$ as 
\[
	S := \bigcup_{s'\in S'} \Aux(s'). 
\]
Note that it suffices to show that 
\begin{equation} \label{eq:meta-complexity-reunion}
	\Ex_{\bg\sim N(0,I_{n+A})}\sbra{\abs{\sup_{s'\in S'}\cbra{\bg\cdot s' + c_{s'}} - \sup_{s\in S} \bg\cdot s }} \leq \frac{\eps}{2}.
\end{equation} 
Indeed, \Cref{eq:guarantee-from-main-theorem,eq:meta-complexity-reunion} immediately imply~\Cref{eq:shifted-shinboku-goal} from the statement of the theorem. 

We establish~\Cref{eq:meta-complexity-reunion} in the remainder of the proof. As in the proof of~\Cref{lemma:cut-union-bound}, note that for any $g\in\R^{n+A}$ we have
\begin{align*}
		\abs{\sup_{s' \in S'} \cbra{g \cdot s' + c_{s'}} - \sup_{s\in S} g\cdot s} 
		&= \abs{\sup_{s' \in S'} \cbra{g \cdot s' + c_{s'}} - \sup_{s'\in S'}\sup_{s \in \Aux(s')} g\cdot s} \\
		&\leq \sup_{s' \in S'} \abs{\cbra{g \cdot s' + c_{s'}} - \sup_{s \in \Aux(s')} g\cdot s} \\
		&= \sup_{s'\in S'}\abs{c_{s'} - \sup_{j\in [A]}\frac{c_{s'}|g_j|}{\mu_A}} \\
		&= \sup_{s'\in S'}\frac{c_{s'}}{\mu_A}\abs{\mu_A - \sup_{j\in[A]} |g_j|} \\
		& \leq \frac{w(T)}{\mu_A} \cdot \abs{\mu_A - \sup_{j\in[A]} |g_j|},
\end{align*}
where we relied on the definition of $\Aux(s')$ in the second equality and on the fact that $c_{s'} \leq w(T)$ in the final inequality. (The latter is readily verified from the proof of~\Cref{lemma:cut-union-bound}.)
Taking expectations with respect to $\bg\sim N(0, I_{n+A})$ then gives 
\[
	\text{L.H.S. of~\eqref{eq:meta-complexity-reunion}} 
	\leq 
	\frac{w(T)}{\mu_A}\Ex_{\bg\sim N(0,I_A)}\sbra{\abs{\mu_A - \sup_{j\in[A]} |g_j|}}
	\leq \frac{w(T)\sigma_A}{\mu_A} = \Theta\pbra{\frac{w(T)}{\log A}}
\]
where penultimate inequality relies on Jensen's inequality (more specifically, on the fact that $\E[\bX] \leq \Ex[\bX^2]^{1/2}$ for a random variable $\bX$) and the final equality follows from~\Cref{prop:gaussian-sup-bounds}. In particular, taking 
\[
	A = \exp\pbra{\Theta\pbra{\frac{w(T)}{\epsilon}}}
\]
immediately implies~\Cref{eq:meta-complexity-reunion}, completing the proof. 
\end{proof}

\subsection{Multiplicative Approximation for Symmetric Processes}
\label{subsec:multiplicative-approximation}

We will establish the following consequence of~\Cref{prop:shinboku}:

\begin{corollary}
\label{cor:multiplicative-approximation}
	Suppose $T\sse\R^n$ is symmetric (i.e. $t \in T$ whenever $-t \in T$) with $w(T) < \infty$. For $\eps \in (0, 0.5)$, there exists another symmetric set $S \sse \R^n$ with $\dim\pbra{\mathrm{span}(S)} \leq c_\eps$, where
	\[
		c_\eps := 2^{\exp\pbra{{O}\pbra{\frac{1}{\eps^3}}}},
	\]	
	such that 
	\[
		\Prx_{\bg\sim N(0, I_n)}\sbra{f_T(\bg) \in \sbra{(1-\eps)f_S(\bg), (1+\eps)f_S(\bg)}} \geq 1-\eps.
	\]
\end{corollary}

In other words, we show that a centered Gaussian process on a symmetric index set can be multiplicatively approximated by a ``low-dimensional'' symmetric process. 
Furthermore, for $n > c_\eps$, we will in fact have 
\[
	|S| \leq 2^{\exp\pbra{{O}\pbra{\frac{1}{\eps^3}}}}.
\]
This will be immediate from the proof of~\Cref{cor:multiplicative-approximation}. (Note that when $n \leq c_\eps$, \Cref{cor:multiplicative-approximation} holds trivially with $T = S$.)

\begin{remark} 
\label{remark:norm-junta}
	Note that~\Cref{thm:norm-junta-theorem} follows immediately from~\Cref{cor:multiplicative-approximation}. 
	To see this, note that any norm $\nu:\R^n\to\R$ can be written as $\nu(x) = f_T(x)$
where $T\sse\R^n$ is the symmetric convex set corresponding to the unit ball of the dual norm $\nu^\circ$ (cf.~\cite{TkoczNotes}). 
	The function $f_S: \R^n\to\R$ is then the approximating norm $\psi: \R^n\to\R$ from the statement of~\Cref{thm:norm-junta-theorem}. 
\end{remark}

\subsubsection{Anti-Concentration of the Supremum for Symmetric Processes}
\label{subsubsec:prof-de}

Towards~\Cref{cor:multiplicative-approximation}, we will require the following anti-concentration lemma: 

\begin{lemma}
\label{lemma:prof-de}	
	Suppose $T\sse\R^n$ is symmetric (i.e. $t \in T$ whenever $-t\in T$), and let $\eps \in (0, 0.5)$. Then 
	\[
		\Prx_{\bg\sim N(0,I_n)}\sbra{\abs{f_T(\bg)} \leq \eps\cdot w(T)} \leq 10\eps.
	\] 
\end{lemma}

Note that unlike the anti-concentration inequality due to Chernozhukov et al.~\cite{CCK-2} (see~\Cref{thm:CCK}), the above lemma does {not} require the index set $T$ to be a subset of $\S^{n-1}$. 
On the other hand, \Cref{thm:CCK} guarantees anti-concentration of $f_T(\cdot)$ everywhere and not just around $0$. 
The proof of~\Cref{lemma:prof-de} will rely on the well-known ``$S$-inequality'' due to Lata\l{}a and Oleszkiewicz~\cite{Latala1999}:

\begin{theorem}[$S$-inequality~\cite{Latala1999}]
\label{thm:S-inequality}
	Let $K \sse \R^n$ be a symmetric convex set, and suppose $S \sse \R^n$ is a symmetric slab with $\Vol(S) = \Vol(K)$, i.e.~ 
	\[
		S = \cbra{x\in\R^n : |x_1| \leq \theta},
	\]
	where $\theta$ is chosen so as to ensure $\Vol(S) = \Vol(K)$. Then for any $t \in [0,1]$, we have 
	\[
		\Vol(tK) \leq \Vol(tS).
	\]
\end{theorem}

\begin{proof}[Proof of~\Cref{lemma:prof-de}]
	Without loss of generality, we can rescale $T$ so that $\|t\| \leq 1$ for all $t\in T$ and furthermore there exists $t_0\in T$ such that $\|t_0\| = 1$. 
	Thanks to symmetry of $T$, we have  
	\[
		f_T(x) \geq |t_0\cdot x| 
		\qquad\text{and so}\qquad 
		\Median\pbra{f_T(\bg)} \geq \Median\pbra{|t_0\cdot \bg|}~\text{for}~\bg\sim N(0,I_n),
	\]
	where $\Median\pbra{\bX}$ denotes the median of the random variable $\bX$. 
	As $\|t_0\|=1$, we have $|t_0\cdot\bg|$ is distributed according to a half-normal distribution and so 
	\begin{equation} \label{eq:median-lb}
		\Median(f_T) := \Median\pbra{f_T(\bg)} \geq 0.67
	\end{equation}
	where $\bg\sim N(0,I_n)$. 
	
	Thanks to our rescaling of $T$ and the fact that $T$ is symmetric, it is easy to check that for $x, y\in\R^n$, 
	\[
		\abs{f_T(x) - f_T(y)} \leq f_T(x-y) = \sup_{t\in T} (x-y)\cdot t \leq \sup_{\|t\|\leq 1} (x-y)\cdot t = \|x-y\|,
	\]
	and so the function $f_T:\R^n\to\R$ is $1$-Lipschitz. 
	\Cref{prop:lipschitz-concentration} then readily implies that 
	\[
		w(T) = \Ex_{\bg\sim N(0,I_n)}[f(\bg)] \leq \Median\pbra{f_T} + 4.
	\]
	(See, for example, Problem~4.2(d) of~\cite{rvh-notes}.) Combining this with~\Cref{eq:median-lb}, we get that $w(T) \leq 10\cdot\Median(f_T)$. This lets us write 
	\begin{align*}
		\Prx_{\bg\sim N(0,I_n)}\sbra{\abs{f_T(\bg)} \leq \eps\cdot w(T)} 
		&\leq  \Prx_{\bg\sim N(0,I_n)}\sbra{\abs{f_T(\bg)} \leq 10\eps\cdot \Median(f_T)} \\
		&= \Vol\pbra{10\eps\cdot\cbra{x : f_T(x) \leq \Median(f_T)}} \\
		& \leq \Vol\pbra{10\eps \cdot \{x : |x_1| \leq c\}} \text{~for $c$ satisfying $\Phi(c)=3/4$}\\
		&\leq 10\eps,
	\end{align*}
	where the penultimate inequality relied on the $S$-inequality (\Cref{thm:S-inequality}). (We relied on the symmetry of the set $T$ to conclude that the set $\cbra{x : f_T(x) \leq \Median(f_T)}$ is a symmetric convex set with Gaussian volume $1/2$.) 
\end{proof}

%
%
%

\subsubsection{Proof of~\Cref{cor:multiplicative-approximation}}
\label{subsubsec:proof-multiplicative-approximation} 

If $n \leq c_\eps$, then the result holds trivially with $T = S$. So we will assume for the remainder of the proof that $n > c_\eps$. 

\paragraph{Step 0: Additive Approximation Suffices.}

By scaling, we can assume that $w(T) = 1$. We will show that there exists a symmetric set $S\sse\R^n$ such that  
\begin{equation} \label{eq:mult-apx-additive-goal}
	|S| \leq 2^{\exp\pbra{{O}\pbra{\frac{1}{\eps^3}}}}
	\qquad\text{and}\qquad 
	\Ex_{\bg\sim N(0,I_n)}\sbra{\abs{f_T(\bg) - f_S(\bg)}} \leq \frac{\eps^3}{40}.
\end{equation}
To see why this suffices, note that by Markov's inequality,
\[
	\Prx_{\bg\sim N(0,I_n)}\sbra{\abs{f_T(\bg) - f_S(\bg)} \geq \frac{\eps^2}{20}} \leq \frac{\eps}{2}. 
\]
Also note that \Cref{lemma:prof-de} implies 
\[
	\Prx_{\bg\sim N(0,I_n)}\sbra{\abs{f_T(\bg)} \geq \frac{\eps}{20}} \geq 1- \frac{\eps}{2}.
\]
\Cref{cor:multiplicative-approximation} then immediately follows. 
The rest of the proof will establish~\Cref{eq:mult-apx-additive-goal}. 

\paragraph{Step 1: Non-Centered Symmetric Approximation.}  

\Cref{prop:shinboku} guarantees the existence of a set $T' \sse T$ with
	\[
		|T'| \leq 2^{\exp\pbra{{O}\pbra{\frac{1}{\eps^3}}}}
	\]
and constants $c_{v}$ for $v\in T$ satisfying $0 < c_v \leq 1$ such that 
	\[
		\Ex_{\bg\sim N(0,I_n)}\sbra{\abs{h(\bg)-f_T(\bg)}} \leq \frac{\eps^3}{160} \quad\text{where}~h(x) := \sup_{v \in T'} v\cdot x + c_v.
	\]
	Inspecting the proof of~\Cref{prop:shinboku}, it is easy to check that applying it to $-T = T$ returns the collection of vectors $-T'$ and associated constants $c'_v$ for $v \in -T'$ where $c'_v = c_{-v}$ and satisfying the following:
	\[
		\Ex_{\bg\sim N(0,I_n)}\sbra{\abs{h'(\bg)-f_T(\bg)}} \leq \frac{\eps^3}{160} \quad\text{where}~h'(x) := \sup_{v \in T'} -v\cdot x + c_v.
	\]
	Note that if $|x-a| \leq \delta_1$ and $|x-b| \leq \delta_2$, then 
	$
		\abs{x - \max\cbra{a, b}} \leq \max \cbra{\delta_1, \delta_2} \leq \delta_1 + \delta_2. 
	$
	This motivates us to define 
	\[
		h''(x) := \sup_{v\in T'}\cbra{v\cdot x + c_v, -v\cdot x + c_v} = \sup_{v\in T' \cup -T'} \cbra{v\cdot x + c_v}. 
	\]
	It follows from the above observation that 
	\begin{align}
		\Ex_{\bg\sim N(0,I_n)}\sbra{\abs{f_T(\bg) - h''(\bg)}} 
		&\leq \Ex_{\bg\sim N(0,I_n)}\sbra{\abs{f_T(\bg) - h(\bg)}} + \Ex_{\bg\sim N(0,I_n)}\sbra{\abs{f_T(\bg) - h'(\bg)}} \nonumber \\
		&\leq \frac{\eps^3}{80}. \label{eq:ha}
	\end{align}
	Set $S_0 := T' \cup -T'$ and note that 
	\[
		|S_0| \leq 2^{\exp\pbra{{O}\pbra{\frac{1}{\eps^3}}}}.
	\]	
	Thus the function $h''(\cdot)$, which approximates $f_T(\cdot)$, is a the supremum of a non-centered Gaussian process on a symmetric set. 

\paragraph{Step 2: Obtaining a Centered Symmetric Approximator.} 

Recall from above that $n > c_\eps$. 
Consequently, we can assume that
$|S_0| + A \leq n$ for $A$ that we will define shortly.
In particular, we can view $\R^n$ as $\R^{|S'|} \times \R^{A'}$ for some $A' \geq A$.

Repeating the argument used to prove~\Cref{cor:shifted-shinboku} on the set $S_0$ with error parameter $\frac{\eps^{2}}{20}$ (and applying a suitable rotation, using the rotation-invariance of the Gaussian distribution), we get that there exists a set $S \sse \R^{|S_0|} \times \R^{A'} = \R^n$ such that $|S| = |S_0|$ and  
	\[
		A = \exp\pbra{O\pbra{\frac{1}{\eps^3}}}
	\]
	such that 
	\begin{equation}
	\label{eq:ho}
		\Ex_{\bg\sim N(0, I_{n+A})}\sbra{\abs{h''(\bg) - f_S(\bg)}} \leq \frac{\eps^3}{80},
	\end{equation}
	and furthermore each $s\in S$ is of the form $(s', \alpha e_i)$ for some $i\in[A]$ and $s'\in S_0$. \Cref{cor:shifted-shinboku} also guarantees that $(s', -\alpha e_i)$ belongs to $S$ whenever $(s', \alpha e_i)$ belongs to $S$. This, combined with the symmetry of $S_0$, implies that the set $S$ is symmetric.
	
\paragraph{Step 3: Putting Everything Together.} 
	\Cref{eq:mult-apx-additive-goal} follows from~\Cref{eq:ha,eq:ho} and the triangle inequality. 
	Furthermore, the set $S$ is symmetric as discussed above and has the desired size as $|S| = |S_0|$. 
	This completes the proof of~\Cref{cor:multiplicative-approximation}.
\qed





\section{Sparsifying Intersections of Narrow Halfspaces}
\label{sec:polytope-sparsification}

As another consequence of~\Cref{prop:shinboku}, we obtain dimension-free sparsification of intersections of halfspaces of bounded (geometric) width:


\theorempolytope*

This can be viewed as a convex set analogue of similar sparsification results obtained for DNFs of bounded width over the Boolean hypercube~\cite{LubyVelickovic:96,Trevisan:04,GMR13,lovett2021decision}, adding to the ``emerging analogy between Boolean functions and convex sets''~\cite{DNS21itcs,DNS22,DNS23-polytope,chen2025lower}. 
We discuss applications to learning theory and property testing in~\Cref{sec:applications}, and consider the tightness of our bounds in~\Cref{sec:lower-bounds}.

Turning to the proof of~\Cref{thm:polytope-sparsification}, we first establish it for the special case of intersections of halfspaces of width \emph{exactly} $r$ for $r \geq 1$:

\begin{lemma}[Special case of~\Cref{thm:polytope-sparsification} when $r_t = r$ for all $t$] \label{lemma:polytope-sparsification-width-1}
Fix $r\geq1$ and suppose $K\sse\R^n$ is a convex set of geometric width $r$, i.e. there exists $T\sse (1/r)\S^{n-1}$ such that 
	\[K = \bigcap_{t\in T} \cbra{x \in \R^n : t\cdot x \leq 1}.\]
	For $0 < \eps \leq 0.5$, there exists a set $L \sse \R^n$ which is an intersection of 
	\[
2^{\exp\pbra{\wt{O}\pbra{\frac{1}{\eps^2}}\cdot r^4}}		~\text{halfspaces}
	\]
	such that $\dG(K, L) \leq \eps$.
\end{lemma}

\begin{proof}[Proof of~\Cref{lemma:polytope-sparsification-width-1}]
		If $\Vol(K) \leq \eps$ then we can take $L = \emptyset$, and if $\Vol(K) \geq 1-\eps$ then we can take $L = \R^n$. So we assume for the remainder of the argument that 
	\begin{equation} \label{eq:vol-k-assumption}
		\Vol(K) \in (\eps, 1-\eps).
	\end{equation}
	
	Next, we will show that 
	\begin{equation} \label{eq:t-prime-bounded-width}
		w(T) \leq 1 + \sqrt{{\frac {2\ln\pbra{\frac{2}{\eps}}}{r^2}}}.
	\end{equation}
	If $w(T) < 1$, then \Cref{eq:t-prime-bounded-width} holds trivially; so suppose $w(T) \geq 1$. 
	Note that $K(x) = 1$ if and only if $f_{T}(x) \leq 1$.
	Let $\{\bX_{t}\}_{t\in T}$ be the canonical Gaussian process on $T$. By~\Cref{eq:vol-k-assumption}, 
	\begin{align}
		\eps \leq \vol(K) &= \Prx_{\bg\sim N(0,I_n)}\sbra{f_{T}(\bg) \leq 1} \nonumber \\ 
		&\leq \Prx_{\bg\sim N(0,I_n)}\sbra{\abs{f_{T}(\bg) - w(T)} \geq \abs{w(T) - 1}} \nonumber \tag{because $w(T)\geq 1$}\\
		&\leq 2\exp\pbra{\frac{-\pbra{w(T) - 1}^2}{2\sup_{t\in T}\Var[\bX_{t}]}}\label{eq:mckinley-ave} \\
		&\leq 2\exp\pbra{\frac{-r^2\pbra{w(T) - 1}^2}{2}} \label{eq:dwight}
	\end{align}
	where~\Cref{eq:mckinley-ave} follows from~\Cref{fact:sup-has-subgaussian-tails} and~\Cref{eq:dwight} relies on the fact that because $T\sse(1/r)\S^{n-1}$, we have $\Var[\bX_t] = 1/r^2$ for all $t\in T$. Rearranging gives \Cref{eq:t-prime-bounded-width} as claimed. 
	
	Let $\eta_1, \eta_2 > 0$ be parameters that we will set later. Since $T$ has bounded Gaussian width, we can now use~\Cref{prop:shinboku} to obtain a set $S \sse T$ with 
	\begin{equation} \label{eq:polytope-same-r-size}
		|S| \leq 2^{2^{O\pbra{\frac{w(T)}{\eta_1}}}}
	\end{equation}
	and constants $\{c_s\}_{s\in S}$ such that 
	\begin{equation} \label{eq:soda-hall}
		\Ex_{\bg\sim N(0,I_n)}\sbra{\abs{f_{T}(\bg) - \sup_{s\in S}\,\cbra{\bg\cdot s + c_s}}} \leq \eta_1.
	\end{equation}
	Motivated by this, we define the \emph{$S$-subspace junta} $J: \R^n\to\R$ to be $J(g) := \sup_{s\in S}\cbra{g\cdot s + c_s}$. (Recall that a function $f: \R^n \to \R$ is said to be an ``$S$-subspace junta'' if there is a subspace $V = \mathrm{span}(S)$ of $\R^n$ such that $f(x)$ depends only on the projection of $x$ on the subspace $V$; see e.g.~\cite{VempalaXiao11,DMN21}.) By Markov's inequality and~\Cref{eq:soda-hall}, 
	\begin{equation} \label{eq:simons-floor-2}
		\Prx_{\bg\sim N(0,I_n)}\sbra{\abs{f_{T}(\bg) - J(\bg)} \geq \sqrt{\eta_1}} \leq \sqrt{\eta_1}.
	\end{equation}
We have by~\Cref{thm:CCK} that 
	\begin{align}
		\Prx_{\bg\sim N(0,I_n)}\sbra{\abs{f_{T}(\bg) - 1} \leq \eta_2} &= \Prx_{\bg\sim N(0,I_n)}\sbra{\abs{f_{rT}(\bg) - r} \leq r\eta_2} \nonumber\\
		&\leq 4\eta_2\cdot r\pbra{1+w(rT)} \nonumber\\
		&= 4\eta_2\cdot r\pbra{1+r\cdot w(T)} \nonumber\\
		&\leq 4\eta_2\cdot r\pbra{1 + r + \sqrt{2\ln\pbra{\frac{2}{\eps}}}} \\ 
		&\leq 4\eta_2\cdot r\pbra{2r + \sqrt{2\ln\pbra{\frac{2}{\eps}}}} \label{eq:fournee-july-4}
	\end{align}
	where the penultimate inequality relied on~\Cref{eq:t-prime-bounded-width} and the final inequality relied on the fact that $r \geq 1$.
	Set
	\[
		\eta_1 = \frac{\eta_2^2}{4}
		\qquad\text{and}\qquad 
\eta_2 = \frac{\eps}{8r\pbra{2r + \sqrt{2\ln\pbra{\frac{2}{\eps}}}}}
	\]
	so as to make \Cref{eq:fournee-july-4} equal to $\eps/2$. 
	It follows from~\Cref{eq:simons-floor-2,eq:fournee-july-4} that 
	\[
		\Prx_{\bg\sim N(0,I_n)}\sbra{\mathbf{1}\cbra{f_T(\bg)\leq 1} \neq \mathbf{1}\cbra{J(\bg) \leq 1}} \leq \eps.
	\]
	Note that $\mathbf{1}\cbra{J(\bg) \leq 1}$ is an intersection of $|S|$ many halfspaces. In particular, note that
	\[
		\frac{1}{\eta_1} = \frac{4}{\eta_2^2} \leq \wt{O}\pbra{\frac{1}{\eps^2}}\cdot r^4 \qquad\text{and so}\qquad  |S| \leq 2^{\exp\pbra{\wt{O}\pbra{\frac{1}{\eps^2}}\cdot r^{4}}}
	\]
	thanks to~\Cref{eq:t-prime-bounded-width,eq:polytope-same-r-size}, which completes the proof.
\end{proof}

We now prove \Cref{thm:polytope-sparsification} using~\Cref{lemma:polytope-sparsification-width-1}.

\begin{proof}[Proof of~\Cref{thm:polytope-sparsification}]
By a limiting argument, we can assume without loss of generality that $K$ is an intersection of finitely many halfspaces. Let $N = |T|$. 
We further note that we can assume every $r_t \geq -\sqrt{2\ln(2/\eps)}$, since otherwise $\Vol(K) < \eps$ and we can trivially $\eps$-approximate $K$ by taking $L$ to be the empty set.  

Let $Q \in \N$ be a parameter that we will soon set; think of $Q$ as a large number. For $t\in T$, define 
\[
	\delta_t := 2r - r_t,  
	\qquad M_t := \exp\pbra{(\delta_t Q)^2}
\]
and set $M := \sup_{t\in T} M_t$. Because $r_t \leq r$ and $r > 0$, it follows that $\delta_t > 0$ for all $t\in T$. 

For each halfspace $H_t = \{x : x\cdot t \leq r_t \}$, define the halfspaces 
\[
	H^{(i)}_t := \cbra{ (x,y) \in \R^n \times \R^M : x\cdot t + \frac{y_i}{\sqrt{2} \cdot Q} \leq 2r} \qquad\text{for}~i\in[M_t]. 
\]
Define the convex body $K' \sse \R^{n+M}$ as 
\[
	K' := \bigcap_{t\in T} \bigcap_{i \in [M_t]} H_t^{(i)}.
\]
We will show that a suitable cross-section of $K'$ is a good approximator to $K$. 
By~\Cref{prop:gaussian-sup-bounds}---in particular, by \Cref{eq:sup-of-gaussians,eq:fine-mean-sup-gaussians})---we have that for any $t\in T$,
\begin{equation} \label{eq:gulpapalooza}
	\mu_t := \Ex_{\by_i \sim N(0, I_M)}\sbra{\sup_{i\in[M_t]} \frac{\by_i}{\sqrt{2}\cdot Q}} = \frac{\sqrt{2\ln(M_t)}(1\pm 4(\ln M_t)^{-1} )}{\sqrt{2}\cdot Q} = \delta_t\pbra{1\pm \frac{4}{\ln M_t} },
\end{equation}
and
\[
	\Varx_{\by_i \sim N(0, I_M)}\sbra{\sup_{i\in[M_t]} \frac{\by_i}{\sqrt{2}\cdot Q}} = \frac{\tau}{Q^2\ln(M_t)} = \frac{\tau}{Q^4\delta_t^2}
\]
for $\tau$ that is at most an absolute constant. 
By Chebyshev's inequality, 
\begin{equation} \label{eq:polytope-chebyshev-app}
	\Prx_{\by_i\sim N(0,I_M)}\sbra{\abs{\sup_{i\in[M_t]} \frac{\by_i}{\sqrt{2}\cdot Q} - \mu_t} \geq \frac{k\sqrt{\tau}}{Q^2\delta_t} } \leq \frac{1}{k^2}. 
\end{equation}
Taking $k = \sqrt{3N/\epsilon}$ ensures that the above probability is at most $\eps/3N$. We now record the first constraint on $Q$, which is that we must ensure that for all $t\in T$ the following holds: 
\[
	\frac{k\sqrt{\tau}}{Q^2\delta_t} = \sqrt{\frac{3N\tau}{\epsilon}}\cdot\frac{1}{Q^2\delta_t} \leq \frac{\eps}{100N},
	\qquad\text{i.e.}\qquad 
	Q \geq \ceil{3^{1/4}\sqrt{100\tau}\pbra{\frac{N}{\epsilon}}^{3/4} \max_{t\in T} \sqrt{\frac{1}{\delta_t}}}. 
\]
\Cref{eq:polytope-chebyshev-app} thus implies that
\begin{equation} \label{eq:polytope-chebyshev-2}
	\Prx_{\by_i\sim N(0,I_M)}\sbra{\abs{\sup_{i\in[M_t]} \frac{\by_i}{\sqrt{2}\cdot Q} - \mu_t} \geq \frac{\eps}{100N} } \leq \frac{\eps}{3N}. 
\end{equation}
Next, viewing $H_t$ as a halfspace in $\R^{n + M}$, note that for any $t \in T$:
\begin{align}
	\dG\pbra{H_t, \bigcap_{i\in[M_t]} H^{(i)}_t} 
	&\leq \frac{\eps}{3N} + \dG\pbra{H_t, \cbra{x : x\cdot t + \mu_t \pm \frac{\eps}{100N} \leq 2r }}   \tag{\Cref{eq:polytope-chebyshev-2}} \\
	&= \frac{\eps}{3N} + \dG\pbra{H_t, \cbra{x : x\cdot t \leq r_t \pm \frac{\eps}{100N} \pm \frac{4(2r-r_t)}{\ln M_t}}} \label{eq:oui}
\end{align}
where we used \Cref{eq:gulpapalooza} to obtain \Cref{eq:oui} from the previous line. 
Now we record the second constraint on $Q$, which comes from requiring that
\[
	\frac{4(2r-r_t)}{\ln M_t} \leq \frac{\eps}{100N};
\]
recalling that $M_t = \exp((\delta_t Q)^2)$, this constraint is
\[
	Q \geq \max_{t \in T} {\frac {20} {\delta_t}} \cdot \sqrt{\frac{N(2r-r_t)}{\eps}}.
\]
Combining this with~\Cref{eq:oui}, we get that 
\begin{equation} \label{eq:sodoi-coffee}
	\dG\pbra{H_t, \bigcap_{i\in[M_t]} H^{(i)}_t} \leq \frac{\eps}{3N} + \frac{4\eps}{100N} \leq \frac{{112}\epsilon}{300N}.
\end{equation}
In particular, a union bound over all $N$ halfspaces immediately implies that 
\begin{equation} \label{eq:k-k'-dist}
	\dG(K, K') \leq \frac{{112}\epsilon}{300}. 
\end{equation}

Next, note that the geometric width of every halfspace in $K'$ is exactly $2r\pbra{1 + (1/(2Q^2))}^{-1/2}$, as $t \in \S^{n-1}$ for all $t\in T$. Note that
\[
	\frac{2r}{\sqrt{1 + (1/2Q^2)}} = \Theta(r)
\]
as $Q\geq 1$. 
It thus follows from \Cref{lemma:polytope-sparsification-width-1} that there exists a set $L' \sse \R^{n+M}$ that is an intersection of
\[
2^{\exp\pbra{\wt{O}\pbra{\frac{1}{\eps^2}}\cdot r^4}}
	~\text{halfspaces}
\]
such that $\dG(K', L') \leq 188\eps/300$. Together with~\Cref{eq:k-k'-dist}, this immediately implies that $\dG(K, L') \leq \eps$ because of the triangle inequality. 

Finally, note that since
\[
	\dG(K, L') = \Ex_{\bx, \by}\sbra{\Indicator\cbra{K(\bx) \neq L'(\bx,\by)}} \leq \eps, 
\]
there exists $y\in\R^M$ such that $\dG(K, L'|_{y}) \leq \epsilon$ where we write $L'|_{y}$ for the cross-section of $L'$ with the coordinates in $\R^M$ set to $y$. Since taking a cross-section cannot increase the number of halfspaces, it follows that $L := L'|_{y}$ has at most as many facets as $L'$, which completes the proof. 
\end{proof}


\newcommand{\conv}{\mathrm{Conv}}
\newcommand{\OPT}{\mathrm{opt}}

\section{Algorithmic Applications}
\label{sec:applications}

In this section we briefly describe some algorithmic applications of our structural results.  More precisely, we obtain polynomial-time agnostic learning results, and constant-query tolerant testing results, for new classes of convex sets under the Gaussian distribution; we do this by combining our new structural result, \Cref{thm:polytope-sparsification} with known algorithms for learning and testing.
Recall that \Cref{thm:polytope-sparsification} tells us that if $K$ is any convex set of geometric width at most $r \geq 1$, then for any $\eps>0$ there is a set $L=L(\eps) \subseteq \R^n$, which is an intersection of at most 
\[
2^{\exp\pbra{r^4 \cdot \wt{O}\pbra{\frac{1}{\eps^2}}}}~\text{halfspaces},
\]
such that $\dG(K, L) \leq \eps$.

The following notation will be helpful:  For $r>0$, let $\conv_r$ denote the class of all convex sets $K \subset \R^n$ that have geometric width at most $r$.  Recall that such a set $K$ is an intersection of (possibly infinitely many) halfspace of the form $\Indicator\cbra{u \cdot x \leq r_u}$, where each $r_u \leq r$.

\subsection{(Agnostically) Learning Convex Sets of Bounded Width}
In \cite{KOS:08} Klivans et al.~gave an algorithm which runs in time $n^{O(\log(k)/\eps^4)}$ and \emph{agnostically learns} intersections of $k$ halfspaces over $\R^n$ to accuracy $\eps$ using only independent labeled examples drawn from $N(0,I_n)$. This means that for any target function $f: \R^n \to \zo$, if there is some intersection of $k$ halfspaces $g: \R^n \to \zo$ which satisfies $\dG(f,g) \leq \OPT$, then with high probability the algorithm outputs a hypothesis $h: \R^n \to \zo$ which satisfies $\dG(f,h) \leq \OPT+\eps$.

For the rest of this section, set the parameter $\alpha(r,\epsilon)$ as 
\[
	\alpha(r,\epsilon):=2^{\exp\pbra{r^4 \cdot \wt{O}\pbra{\frac{1}{\eps^2}}}}.
\]
By \Cref{thm:polytope-sparsification}, for any function $f' \in \conv_r$ there is an intersection of at most $\alpha(r,\eps/2)$ halfspaces, which we denote $g$, which is $\eps/2$-close to $f'$. 
The \cite{KOS:08} result (applied with its ``$\eps$'' parameter set to $\eps/2$) thus immediately yields the following:

\begin{corollary} \label{cor:agnostically-learn-samples}
Let $r\geq 1$. There is an algorithm which runs in time $n^{O(\log({\alpha(r,\eps/2)})/\eps^4}$ and agnostically learns the class $\conv_r$ to accuracy $\eps$ using only independent labeled examples drawn from $N(0,I_n)$.
\end{corollary}

Diakonikolas et al.~\cite{DKKTZ23} have recently shown that if the learning algorithm is allowed to make black-box queries to the unknown target function $f$, then intersections of $k$ halfspaces can be agnostically learned in time $\poly(n) \cdot 2^{\poly(\log(k)/\eps)}$.  This immediately gives the following:

\begin{corollary} \label{cor:agnostically-learn-queries}
Let $r\geq 1$. There is an algorithm which runs in time $\poly(n) \cdot 2^{\poly(\log({\alpha(r,\eps/2)})/\eps)}$ and agnostically learns the class $\conv_r$ to accuracy $\eps$ under the standard $N(0,I_n)$ distribution using black-box queries.
\end{corollary}

For constant $\eps$, \Cref{cor:agnostically-learn-samples} achieves $\poly(n)$-time agnostic learning for convex sets of any constant geometric width, and \Cref{cor:agnostically-learn-queries} achieves $\poly(n)$-time agnostic learning (with queries) for convex sets of geometric width $c\cdot(\log\log n)^{1/4}$ for an absolute constant $c > 0$.

\subsection{(Tolerantly) Testing Convex Sets of Bounded Width}

In \cite{DMN21} De et al.~gave a \emph{tolerant testing} algorithm for the class of intersections of $k$ halfspaces under the standard $N(0,I_n)$ Gaussian distribution, which uses $k^{\poly(\log(k)/\eps)}$ black-box queries.   This means that for any target function $f: \R^n \to \zo$, for any desired input parameters $0\leq \eps_1<\eps_2$ with $\eps_2-\eps_1=\eps$, 

\begin{itemize}

\item If there is some intersection of $k$ halfspaces $g: \R^n \to \zo$ which satisfies $\dG(f,g) \leq \eps_1$ then with high probability the algorithm outputs ``accept,'' and

\item If every intersection of $k$ halfspaces $g$ has $\dG(f,g) \geq \eps_2$ then with high probability the algorithm outputs ``reject.''

\end{itemize}

We emphasize that the query complexity of the \cite{DMN21} algorithm is completely independent of the ambient dimension $n$.

Combining the \cite{DMN21} testing algorithm with \Cref{thm:polytope-sparsification}, we immediately get  that the class $\conv_r$ is tolerantly testable with query complexity just dependent on $r$ and independent of the ambient dimension $n$. In particular, we have the following result:
\begin{theorem}
Let $r\geq 1$. For any input parameters $0 \le \epsilon_1 < \epsilon_2$, there is an algorithm which makes $k^{\poly(\log(k)/\eps)}$ queries
to  the target function $f: \mathbb{R}^n \rightarrow \{0,1\}$ and has the following guarantee: 
\begin{itemize}
\item If there is a function $g \in \conv_r$ such $\dG(f,g) \leq \eps_1$ then with high probability the algorithm outputs ``accept,'' and
\item If for every $g \in \conv_r$,  $\dG(f,g) \geq \eps_2$ then with high probability the algorithm outputs ``reject.''
\end{itemize}
Here $\epsilon=\epsilon_2- \epsilon_1$ and $k= 2^{\exp\pbra{r^4 \cdot \wt{O}\pbra{\frac{1}{\eps^2}}}}.$
\end{theorem}


\section{Lower Bounds}
\label{sec:lower-bounds}

In this section we establish some lower bounds which complement the upper bounds given in \Cref{thm:GP-sparsification}, \Cref{thm:norm-junta-theorem} (and consequently~\Cref{cor:multiplicative-approximation}) and \Cref{thm:polytope-sparsification}. 
Before giving these results in \Cref{sec:lower-main}, we first briefly discuss the ``non-proper'' nature of our upper bounds in \Cref{sec:proper-impossible}. 

\subsection{The Impossibility of ``Proper'' Approximation} \label{sec:proper-impossible}

In this section we observe that dimension-independent ``proper'' variants of our approximation results  \Cref{thm:GP-sparsification} (sparsifying the $f_T$ function), \Cref{thm:norm-junta-theorem} (sparsifying a norm), and \Cref{thm:polytope-sparsification} (sparsifying a polytope with bounded geometric width) cannot exist.

Let $T$ be a bounded subset of $\R^n$, so $w(T) < \infty$, and consider the corresponding centered Gaussian process $\{\bX_t\}_{t \in T}$ and the associated function $f_T(x) = \sup_{t \in T} t \cdot x.$
\Cref{thm:GP-sparsification} states that $f_T$ is $\eps$-approximated in $L_1$, under $N(0,I_n)$, by a function $g(x) = \sup_{s \in S} \{s \cdot x +c_s\}$ where $S \subseteq T$ satisfies $|S| \leq  \pbra{1/\eps}^{\exp\pbra{O\pbra{w(T)/\epsilon}}}$; so the size of $S$ depends only on $w(T)$ and $\eps$, and is completely independent of the ambient dimension $n$.

Even though $S$ is a subset of $T$, the approximating function $g$ is not a ``proper'' approximator because of the shifts $c_s$. It is natural to wonder whether those shifts are necessary: Can $f_T$ be $\eps$-approximated in $L_1$ by a function of the form $\sup_{s \in S} \{s \cdot x\}$, where $S$ is a subset of $T$ of size $O_{\eps,w(T)}(1)$, independent of $n$?  The answer to this question is no, as shown by the following simple example:

\begin{example} \label{ex:L1}
Let $a(n) := \Ex_{\bg \sim N(0,I_n)}[\sup \{\bg_1,\dots,\bg_n\}]$ be the expected supremum of $n$ i.i.d.~standard Gaussians. 
Sharpening the bound from~\Cref{prop:gaussian-sup-bounds},  
it is well known (see e.g. Exercise~5.1 of \cite{rvh-notes}) that $a(n) = (1\pm o_n(1)) \sqrt{2 \ln n}.$
Let $T=\{{\frac 1 {a(n)}}e_1,\dots,{\frac 1 {a(n)}}e_n\}$, so the Gaussian width of $T$ is $w(T)=1$. 
Now consider any candidate ``proper'' approximator $f_{S}(x)$ for $f_S(x)$ which is obtained by choosing a subset $S \subset T,$ $|S|=m$ of $m=n^{1-c}$ of the $n$ vectors in $T$, where $0<c<1$ is any fixed constant. Without loss of generality we may take $S$ to be the first $m$ vectors in $T$, so $f_{S}(x) = {\frac 1 {a(n)}} \sup_{1 \leq i \leq m} \{x_i\}$ whereas $f_T(x) =  {\frac 1 {a(n)}} \sup_{1 \leq i \leq n} \{x_i\}.$
Since $a(m)=(1\pm o(1))\sqrt{2 \ln m}=(1\pm o(1))\sqrt{2(1-c)\ln n}$, we have $\E\sbra{\sup_{t \in S} \bX_t}=(1\pm o(1))\sqrt{1-c}$, and so applying \Cref{fact:sup-has-subgaussian-tails} to the canonical Gaussian processes $\{\bX_t\}_{t \in T}$ and $\{\bX_s\}_{s \in S}$, we get that 
\begin{align*}
\Prx_{\bg \sim N(0,I_n)}\sbra{f_T(\bg) \in \sbra{1-o(1),1+o(1)}} &\geq 1 - o(1), \\
\Prx_{\bg \sim N(0,I_n)}\sbra{f_{S}(\bg) \in \sbra{\sqrt{1-c}-o(1),\sqrt{1-c}+o(1)}} &\geq 1 - o(1),
\end{align*}
so $\Ex_{\bg \sim N(0,I_n)}\sbra{|f_{T}(\bg) - f_{S}(\bg)|} \geq 1-\sqrt{1-c} - o(1) = \Omega(1).$ 
Thus, even if $w(T)=1$, the size of $S \subset T$ which is needed for $f_S(x)$ to $\eps$-approximate $f_T(x)$ in $L_1$ under $N(0,I_n)$ is at least $n^{1-c(\eps)}$, where $c(\eps) \to 0$ as $\eps \to 0.$
\end{example}

Turning to \Cref{thm:polytope-sparsification}, recall that that result shows that any convex set $K \subset \R^n$ with geometric width at most $r$ can be $\eps$-approximated by an intersection $L$ of at most 
\[
2^{\exp\pbra{r^{4}\cdot \wt{O}\pbra{\frac{1}{\eps^2}}}}
	~\text{halfspaces},
\]
and that moreover the halfspaces constituting $L$ in our construction may not be supporting halfspaces of $K$.  We remark that in the analogous context of width-$w$ CNF approximation over $\zo^n$, the state-of-the art result of Lovett et al.~\cite{lovett2021decision} gives an $\eps$-approximating width-$w$ CNF of   size $s=(2+ {\frac 1 w} \log(1/\eps))^{O(w)}$ which is obtained by keeping $s$ (carefully chosen) clauses in the original CNF and discarding the rest of them.  Thus it is natural to wonder whether there is an analogue of \Cref{thm:polytope-sparsification} in which (i) each halfspace constraint in the approximator is required to be one of the original supporting halfspaces of $K$, and yet (ii) the sparsity of the approximator is independent of $n$.  It turns out that this is impossible, as witnessed by the following simple example which is reminiscent of \Cref{ex:L1}:

\begin{example} \label{ex:polytope}
Let $K$ be a convex set in $\R^{n+1}$ which is the intersection of $n$ halfspaces, the $i$-th of which is 
\begin{equation} \label{eq:ex}
x_0 + {\frac 1 {a(n)}}x_i \leq \sqrt{1 + {\frac 1 {a(n)^2}}},
\end{equation}
so the geometric width of $K$ is 1.  Similar to \Cref{ex:L1}, for any constant $0<c<1$ let $L$ be the intersection of $m=n^{1-c}$ of these $n$ halfspaces; without loss of generality we can suppose the $m$ halfspaces are given by \Cref{eq:ex} for $i=1,\dots,m$.  Now, observe that $x \in K$ iff $x_0 + {\frac 1 {a(n)}} \sup \{x_1,\dots,x_n\} \leq \sqrt{1 + {\frac 1 {a(n)^2}}}$. Since by \Cref{fact:sup-has-subgaussian-tails} we have $\Pr[\sup\{\bg_1,\dots,\bg_n\}=(1\pm o(1))a(n)]=1-o(1)$, for a $1-o(1)$ fraction of outcomes of $\bg_1,\dots,\bg_n$, we have that $\bg=(\bg_0,\dots,\bg_n) \in K$ iff 
\[
\bg_0 \leq \sqrt{1 + {\frac 1 {a(n)^2}}} - {\frac {(1\pm o(1))a(n)}{a(n)}} = \pm o(1).
\] 
Turning to $L$, we have that $x \in L$ iff $x_0 + {\frac 1 {a(n)}} \sup \{x_1,\dots,x_m\} \leq \sqrt{1 + {\frac 1 {a(n)^2}}}$. By \Cref{fact:sup-has-subgaussian-tails} we have $\Pr[\sup\{\bg_1,\dots,\bg_m\}=(1\pm o(1))a(m)]=1-o(1)$, so for a $1-o(1)$ fraction of outcomes of $\bg_1,\dots,\bg_n$, we have that $\bg=(\bg_0,\dots,\bg_n) \in L$ iff 
\[
\bg_0 \leq \sqrt{1 + {\frac 1 {a(n)^2}}} -
{\frac {(1\pm o(1))a(m)}{a(n)}}
= 1 - (1\pm o(1))\sqrt{1-c} \pm o(1).
\]
It follows that for any constant $0<c<1$, we have $\Vol(K \triangle L)$ is at least an absolute constant, so $L$ cannot be an $\eps$-approximator of $K$ for sufficiently small constant $\eps$.
\end{example}

This example can be adapted to the context of \Cref{thm:norm-junta-theorem} (norm approximation); we leave the details to the interested reader.

\subsection{Lower Bounds 
} \label{sec:lower-main}

In this section we give lower bounds on sparsification.
To begin, a straightforward geometric argument shows that any intersection of halfspaces which $\eps$-approximates the unit ball $\{(x_1,x_2): x_1^2 + x_2^2 \leq 1\}$ in $\R^2$ must use at least $(1/\eps)^{\Omega(1)}$ halfspaces. This simple example already shows that approximating bounded-geometric-width convex sets with only $\polylog(1/\eps)$ halfspaces is impossible (in contrast with known results for Boolean CNF sparsification, as discussed in \Cref{sec:discussion-applications}). 

In this context it is instructive to recall the upper bound from \cite{DNS23-polytope} which we stated as \Cref{thm:polytope-approximation-ub}.  This result tells us that for any fixed value of $n$, if $\eps$ is allowed to grow sufficiently small relative to $n$, then for any convex set $K \subseteq \R^n$ (not necessarily of bounded width), $(n/\eps)^{O(n)}$ halfspaces suffice for $\eps$-approximation; this is $\poly(1/\eps)$ many halfspaces for any fixed value of  $n$.  But this is not the end of the story: if we take $\eps \to 0$ and allow the dimension $n$ to vary with $\eps$, then it is possible to give a lower bound which is \emph{exponential} in $1/\eps$:

\begin{theorem} \label{thm:current-lower-bound}
For all $\eps > 0$, for all sufficiently large $n \geq n(\eps)$, there is a convex body $K \subset \R^n$ of geometric width 1, such that any $\eps$-approximating convex set $L \subset \R^n$ for $K$ must be an intersection of at least $2^{\Omega(1/\eps)}$ many halfspaces.
\end{theorem}

\begin{proof}
The argument is a reduction to a recent lower bound of \cite{DNS23-polytope} (\Cref{thm:ball-approximation-lb}) showing that any intersection of halfspaces that $\Theta(1)$-approximates the origin-centered ball of radius $\approx \sqrt{n}$ must use $2^{\Omega(\sqrt{n})}$ halfspaces. 

Given $\eps > 0$, let $n = \lceil 1/\eps^2 \rceil.$  For convenience we work in $\R^{n+1}$. Let $K \subset \R^{n+1}$ be the body defined as
\begin{equation} \label{eq:Kdef}
\bigcap_{u=(u_1,\dots,u_n): \|u\|_2 = 1}
\cbra{x \in \R^{n+1} \ : {\frac {u_1 x_1 + \cdots + u_n x_n}{\sqrt{n}}} + x_{n+1} \leq 1 + {\frac 1 n}}.
\end{equation}
It is clear that the geometric width of $K$ is exactly 1.

Suppose that $L$ is an intersection of halfspaces such that $\Vol(L,K) \leq c/\sqrt{n}$ (here $c>0$ is a suitably small absolute constant that we will specify below). For $\rho \in \R$, write $K_\rho$ to denote the $n$-dimensional cross-section of $K$, 
\[
	\bigcap_{u=(u_1,\dots,u_n): \|u\|_2 = 1}
	\cbra{x \in \R^{n} \ : {\frac {u_1 x_1 + \cdots + u_n x_n}{\sqrt{n}}} + \rho \leq 1 + {\frac 1 n}},
\]
obtained by fixing the last coordinate $x_{n+1}$ to $\rho$, and likewise write $L_\rho$ for the corresponding cross-section of $L$. 
Using the fact that the pdf of the standard $N(0,1)$ Gaussian is at least 0.3 everywhere on $[-1/(2\sqrt{n}),1/(2\sqrt{n})]$, there must exist an outcome of $\rho \in [-1/(2\sqrt{n}),1/(2\sqrt{n})]$ such that $\Vol_n(K_\rho \ \triangle \ L_\rho) \leq 10c$ (in fact, most outcomes of  $\rho \in [-1/(2\sqrt{n}),1/(2\sqrt{n})]$ must have this property). 

Fix such an outcome of $\rho \in [-1/(2\sqrt{n}),1/(2\sqrt{n})]$. The set $K_\rho \subset \R^n$ is equivalent to
\[
	\bigcap_{u=(u_1,\dots,u_n): \|u\|_2 = 1}
	\cbra{x \in \R^{n} \ : {u_1 x_1 + \cdots + u_n x_n \leq \sqrt{n} \cdot \pbra{1 + {\frac 1 n} - \rho}}},
\]
which is an origin-centered ball of radius between $\sqrt{n} - 1$ and $\sqrt{n} + 1$.
Choosing $c=\kappa/10$ for the value $\kappa$ from \Cref{thm:ball-approximation-lb} and applying that theorem, we get that $L_\rho$ must be an intersection of $2^{\Omega(\sqrt{n})} = 2^{\Omega(1/\eps)}$ halfspaces, so $L$ must also be an intersection of $2^{\Omega(1/\eps)}$ halfspaces, and \Cref{thm:current-lower-bound} is proved.
\end{proof}

We remark that \Cref{thm:ball-approximation-lb} directly gives that for a suitable absolute constant $\kappa > 0$, for any integer $w \geq 1$, there is a convex set $K$ of geometric width $w$ (namely the origin-centered ball of radius $w$ in $\R^n$, where $n \geq w^2$) such that any $\kappa$-approximating convex set $L$ for $K$ must be an intersection of at least $2^{\Omega(w)}$ many halfspaces.

We also remark that the lower bound of \Cref{thm:current-lower-bound} easily yields a similar lower bound on sparsifying suprema of Gaussian processes: 

\begin{corollary} \label{cor:lb-sparsifying-suprema}
For all $\eps > 0$, for all sufficiently large $n \geq n(\eps)$,
there is a set $T\sse\R^n$ with $w(T) =1$ such that any set $S$ and 
	and collection of associated real constants $\{c_s\}_{s\in S}$ satisfying 
	\begin{equation} \label{eq:lb-supremum}
		\Ex_{\bg\sim N(0,I_n)}\sbra{\abs{f_T(\bg) - \sup_{s\in S} \, \{\bg\cdot s + c_s\}}} \leq \eps
	\end{equation}
must have $|S| = 2^{\widetilde{\Omega}(1/\sqrt{\eps})}.$
\end{corollary}

The idea behind \Cref{cor:lb-sparsifying-suprema} is that if there existed a sparse set $S$ of size $2^{\wt{o}(1/\sqrt{\eps})}$ with associated constants $\{c_s\}_{s \in S}$ satisfying \Cref{eq:lb-supremum}, then following the proof of \Cref{thm:polytope-sparsification} we would get an approximating convex set $L$ for the set $K$ in \Cref{eq:Kdef} with $2^{o(1/\eps)}$ halfspaces, contradicting \Cref{thm:current-lower-bound}.
We leave the details to the interested reader.

\section*{Acknowledgements}

A.D.~is supported by NSF grants CCF-1910534 and CCF-2045128. 
S.N.~is supported by NSF grants CCF-2106429,
CCF-2211238, CCF-1763970, and CCF-2107187. 
R.A.S.~is supported by NSF grants CCF-2106429
and CCF-2211238. 
This work was partially completed while some of the authors were visiting the Simons Institute for the Theory of Computing. 

\appendix

\bibliography{allrefs}
\bibliographystyle{alphaurl}

\end{document}